%% file: main.tex
\documentclass{article}
\pdfoutput=1
\usepackage[total={6in,8in}]{geometry}

\usepackage{times}
\usepackage[round]{natbib}
\usepackage[utf8]{inputenc} % allow utf-8 input
\usepackage[T1]{fontenc}    % use 8-bit T1 fonts
\usepackage{hyperref}       % hyperlinks
\usepackage{url}            % simple URL typesetting
\usepackage{booktabs}       % professional-quality tables
\usepackage{amsfonts}       % blackboard math symbols
\usepackage{nicefrac}       % compact symbols for 1/2, etc.
\usepackage{microtype}      % microtypography
\usepackage{amsmath}
\usepackage{amsthm}
\usepackage{amssymb}
\usepackage{mathtools}
\usepackage{graphicx}
\usepackage{caption}
\usepackage[ruled]{algorithm2e}

\newtheorem{theorem}{Theorem}
\newtheorem{lemma}{Lemma}
\newtheorem{proposition}{Proposition}
\newtheorem{corollary}{Corollary}
\newtheorem{definition}{Definition}

\newtheorem*{theorem*}{Theorem}
\newtheorem*{lemma*}{Lemma}
\newtheorem*{proposition*}{Proposition}
\newtheorem*{corollary*}{Corollary}

\usepackage{color}

\title{On the Approximation Capabilities of ReLU Neural Networks and Random
ReLU Features}
\author{
  Yitong Sun \\
  Department of Mathematics \\
  University of Michigan \\
  Ann Arbor, MI, 48109\\
  \texttt{syitong@umich.edu} \\
  \and
  Anna Gilbert\\
  Department of Mathematics\\
  University of Michigan\\
  \texttt{annacg@umich.edu} \\
  \and
  Ambuj Tewari \\
  Department of Statistics\\
  University of Michigan\\
  \texttt{tewaria@umich.edu} \\
}

\begin{document}
    \maketitle
    \begin{abstract}
      \input{abstract}
    \end{abstract}
        \input{introduction}
        \input{prelim}
        \input{universality-relu-feature}
        \input{gen-error}
        \input{barron}
        \input{experiments}
        \input{conclusion}

        \clearpage

\input{main.bbl}
        \clearpage
        \onecolumn
        \setcounter{page}{1}
        \appendix
        \input{pf_universality}
        \input{pf_learn-rate}
        \input{pf_sep_rkhs}
        \input{pf_barron}

        \input{pf_multi}
        \clearpage
        \input{extra_img}
\end{document}

%% file: abstract.tex
%!TEX root = main.tex

%\ambuj{revise once main paper has converged}
We study the approximation properties of random ReLU features through
their reproducing kernel Hilbert space (RKHS).
We first prove a universality theorem for the RKHS induced by random features
whose feature maps are of the form of nodes in neural networks.
The universality result
implies that the random ReLU features method is a universally consistent learning algorithm.
We prove that despite the universality of the RKHS induced by the random ReLU features,
composition of functions in it generates substantially more
complicated functions that are harder to approximate than those functions simply in the RKHS.
We also prove that such composite functions can be efficiently approximated by
multi-layer ReLU networks with bounded weights.
This depth separation result shows that the random ReLU features models
suffer from the same weakness as that of shallow models.
We show in experiments that the performance of random ReLU features is comparable to that of
random Fourier features and, in general, has a lower computational cost.
We also demonstrate that when the target function is the composite function as described in
the depth separation theorem, 3-layer neural networks indeed outperform
both random ReLU features and 2-layer neural networks.

%% file: introduction.tex
%!TEX root = main.tex
\section{Introduction}
\label{sec:intro}
Random features methods have drawn researchers' attention since \cite{Rahimi2008}
showed their connections to kernel methods. Examples include random Fourier
features with Gaussian feature distribution approximating Gaussian kernels,
and random binning features approximating Laplacian kernels. In supervised learning
tasks, a linear regression function or classifier is learned on top of these random features.

When random features are used as standalone learning methods;
that is, no kernels are chosen in advance,
we want flexibility in choosing the feature
map and feature distribution for such purposes as to lower computational
cost or to inject prior knowledge. In this paper,
we consider the random features using the following form of feature maps
\[
    \sigma(\omega\cdot x+b)\,,
\]
where $ (\omega,b) $ are the parameters to be randomly chosen according to a certain distribution.
Linear combinations of this type of random features give solutions to supervised learning tasks in the form of
neural networks. This provides a foundation for the comparision between kernel methods and neural networks.

We are particularly interested in using ReLU as the feature map, because ReLU activation
nodes are widely used in deep neural networks for their computational advantages.
Note that we do not need to evaluate the derivatives of feature maps in random features methods,
so the advantage of ReLU in preventing vanishing or exploding gradient does not play a role in our case.
However, using ReLU in random features methods may still provide other computational advantages
over sigmoidal or sinusoidal feature maps. We will see in our experiments
that random ReLU features models output sparser feature vectors,
since each node either outputs 0 or identity.

To justify the use of a certain random feature in supervised learning tasks,
we need to answer the question: \emph{For a given
feature map and a given feature distribution, under what conditions it is
guaranteed that there exists a linear combination of randomly chosen
features that is a good approximator to the target continuous function?}

In general, a supervised learning algorithm or its hypothesis class is called universal, if the hypothesis class
is dense in the space of continuous functions.
From classic results we know that neural networks of any
non-polynomial activation functions are universal \citep{Leshno1993}.
Even though linear combinations of non-polynomial activation functions form
a dense subset in the space of continuous functions, it requires
appropriate setup of parameters inside the activation functions
to construct the approximator.
And thus it does not imply that by randomly sampling inner weights,
with high probability there exists a linear combination of randomly-chosen
nodes approximating the target continuous function well.

In this work, we study the question asked above through the universality of
the reproducing kernel Hilbert spaces (RKHS) corresponding to the random features.
Random Fourier features with Gaussian feature distribution and random binning
features are potentially good choices in supervised learning
tasks because their corresponding RKHS are universal \citep{Micchelli2006}.
When the feature map is chosen to be $ \mathrm{ReLU}(\omega\cdot \mathbf{x} + b) $ where the parameters
vector $ (\omega,b) $ obeys the standard Gaussian over $ \mathbb{R}^{d+1} $
or the uniform distribution over $ \mathbb{S}^d $, the induced kernel is one of the
arccos kernels and can be written in a closed form \citep{Cho2009}.
When some other distributions or feature maps are considered,
there may not be good closed forms for the induced kernels.
\cite{Bach2017} analyzed the approximation
property of the arccos kernels. Proposition~3 in \cite{Bach2017}
implies the universality of arccos kernels by explicitly constructing
functions in the RKHS of arccos kernels as approximators to Lipschitz functions.
However this approach relies on the closed form of arccos kernels and thus may
not work for random ReLU features with feature distribution other than Gaussian
or uniform over the unit sphere.
In this paper, we use tools from functional analysis and provide
a set of broad sufficient conditions on the feature map $ \sigma $
and distribution of $ (\omega,b) $ for the universality
of the corresponding RKHS.

Based on the universality result, the approximation
capability of a random feature can be characterized by how well it
approximates functions in the RKHS. For this part, by applying
\cite{Bach2017a}'s result, we reproduce \cite{Huang2006a}'s result
for bounded feature maps with a much simpler proof, and extend it to
unbounded but admissible random features (See details of the
admissibility condition in Section~\ref{sec:universality}).
These results provide an answer to the question we asked above on the
existence of good approximators when using random features methods.
We further show that the good hypothesis can always be found by solving
the constrained empirical risk minimization problem on top of
random ReLU features.

To further study the RKHS induced by the random ReLU feature,
we investigate the composition of functions from its RKHS.
\cite{Eldan2016} show that 3-layer ReLU networks with polynomial many nodes can express functions
that cannot be approximated by any 2-layer ReLU networks with polynomially many
($\mathrm{poly}(d)$) nodes.
\cite{Daniely2017} constructs examples over $\mathbb{S}^{d-1}\times\mathbb{S}
^{d-1}$ that can be well approximated by 3-layer ReLU networks with
number of nodes and weights less than $\mathrm{poly}(d)$, but not by any 2-layer ReLU networks
with much more ($\Omega(\exp(d))$) nodes and much larger weights.
\cite{Lee2017} shows a similar depth separation result for functions in
Barron's class \citep{Barron1993} and their compositions.
Following these works, we show a depth separation result for functions in the RKHS
induced by the random ReLU feature. It shows that compositions of
functions from the RKHS generate substantially more complicated functions that
are hard to be approximated by functions in the RKHS.
We also prove that compositions of functions from the RKHS can be approximated by a multi-layer ReLU
network, with all weights bounded by constants depending on the RKHS norm of components of target functions.
The depth separation result and the multi-layer approximation result together
suggest that random ReLU features suffer from the common weakness of
shallow models compared to deep ones.

To verify the performance of the random ReLU features method,
we compared them to the random Fourier features method on several
real and synthetic datasets. It is confirmed in our experiments
that the random ReLU features method can achieve similar performance
with random Fourier features method with lower computational cost.
We also designed synthetic datasets according to the construction in the depth separation
result and use them to demonstrate the difference among the performance of random ReLU features,
2-layer neural networks and 3-layer neural networks.
The experiment clearly shows the limit of shallow models.

The overall contributions of this paper are summarized as follows.
\begin{enumerate}
  % \item In analogy with Barron's class, we define a class of functions
  % $ \Lambda_R $ using the transformation induced by the ReLU activation function,
  % and prove a quantitative approximation result for functions
  % in $ \Lambda_R $ using ReLU neural networks (Theorem~\ref{thm:main}).
  % In our approximation theorem,
  % both outer and inner weights are controlled.
  % \item We prove that any composition of functions in $ \Lambda_R $ can
  % be approximated by multi-layer ReLU networks, with
  % the norm of weight matrices bounded by constants that depend on $ R $ and
  % input/output dimensions of these functions.
  % (Theorem~\ref{thm:multi}).
  % We use the results of \cite{Eldan2016} establishing a separation theorem
  % showing the essential difference in capacities between $ \Lambda_R $
  % and compositions of functions
  % in $ \Lambda_R $ when $ R=O(\mathrm{poly}(d)) $
  % (Proposition~\ref{prop:separation}).
  \item We establish sufficient conditions for the universality of
  kernels induced by a broad class of random features (Theorem~\ref{thm:universality}).
  Based on this result we are able to prove the universality of
  random features in admissible cases (Corollary~\ref{cor:opt-universal}
  and \ref{cor:bounded}).
  \item We describe the random ReLU features
  method (Algorithm~\ref{alg:ran-relu}) and show that
  it is universally consistent.
  We compare the performance of random ReLU features to random Fourier
  features and confirm the advantages of random ReLU features in computational
  cost.
  \item We prove that compositions of functions in the RKHS induced by
  the random ReLU features generate more complicated functions (Proposition~\ref{prop:separation}). And such
  composite functions can be efficiently approximated by multilayer ReLU networks (Proposition~\ref{prop:multi}).
  Our experiments confirm the gap between the performance of shallow and
  deep models, and show that the good deep approximator can be found
  by simply running stochastic gradient descent.
\end{enumerate}

In Section~\ref{sec:prelim},
we describe random features and reproducing kernel Hilbert spaces,
and define the notations we use in the paper.
The universality of the random ReLU features is given in
Section~\ref{sec:universality}. We describe a simple random ReLU features method
and show its universal consistency in Section~\ref{sec:learn-rate}.
The depth separation result of RKHSs induced by the random ReLU feature and
the approximation by multilayer ReLU networks are presented
in Section~\ref{sec:barron}.
The performance of random ReLU features in experiments and their strength and weakness
are discussed in Section~\ref{sec:experiments}. All the proofs and
extra experiment results can be found in the appendices.

%% file: prelim.tex
%!TEX root = main.tex
\section{Preliminaries on Random Features}
\label{sec:prelim}
Throughout the paper, we assume that $ \mathcal{X} $ is a
closed subset contained in the ball centered at the origin
of $ \mathbb{R}^d $ with radius $ r $.
Let $ C(\mathcal{X}) $ denote the space of all continuous functions on $ \mathcal{X} $ equipped with
the supremum norm. When a subset of $ C(\mathcal{X}) $ is dense,
we call it universal.

For any positive symmetric kernel function $ k(x,x') $, we call $
\phi:\mathcal{X}\to H $, where $ H $ is a Hilbert space,
a feature map of $ k $, if
\begin{equation}
  \label{eq:feature}
  k(x,x')=\langle\phi(x),\phi(x')\rangle_H\,.
\end{equation}
Feature representations of kernels
are very useful for understanding the approximation property of the RKHS, and
also for scaling up kernel methods to large data sets. In practice, one can
choose a kernel function for the problem first, and then pick up a feature map
based on some transformation of the kernel function.
The process can also be reversed. One can design a map $ \phi $ from
$ \mathcal{X} $ to a Hilbert space $ H $ first, and then
define the kernel function by Equation~\ref{eq:feature}. A class of useful
feature maps chooses $ H $ to be $ L^2(\Omega,\omega,\mu) $,
where $ \mu $ is a probability distribution over the parameter space
$ \Omega $. We call the pair $(\phi,\mu)$ a random feature.
The corresponding kernel function and the RKHS are denoted by $ k_{\phi,\mu} $
and $ (\mathcal{H}_{\phi,\mu},\Vert\cdot\Vert_{\phi,\mu}) $, respectively. Any function $ f $ in
$ \mathcal{H}_{\phi,\mu} $ can be described by
\begin{equation*}
 f(x)=\int_{\Omega}\phi(x;\omega)g(\omega)\,\mathrm{d}\mu(\omega)\,,
\end{equation*}
for some $ g\in L^2(\Omega,\omega,\mu) $. The function $ g $ in the
representation of $ f $ is not unique, and
$ \Vert f\Vert_{\phi,\mu} $ equals the infimum of the $ L^2 $
norm of all such $ g $'s.
Because $ \mu $ is a probability measure, we can approximate $f$ using
  $\frac{1}{N}\sum_{i=1}^N \phi(x;\omega_i)g(\omega_i)$
with $ \omega_i $ sampled independently according to $ \mu $.
In a supervised learning task, the coefficients $ g(\omega_i) $
can be determined by a training process minimizing the empirical risk with
respect to some loss function; this is known as the random features method.
The generalization error of random Fourier features methods has been studied by
\cite{Rudi2017} and \cite{Sun2018b}. They have been shown to be able to achieve
fast learning rates with number of nodes far fewer than the sample size.
The performance of random features in practical problems such as speech
recognition has also been investigated; see \cite{Huang2014a}.
Note that in the entire process of applying a random features method,
the kernel function does not show up at all.

The feature map we considered in this work is of the form
$ \phi(x;\omega,b) = \sigma(\omega\cdot x+b) $ where $ \sigma $ is a non-linear
continuous function on $ \mathbb{R} $. We refer this feature map as
of \emph{neural network type}.
To simplify our notations, we denote by $ (x,1) $ the concatenation of the
$ d $-dimensional vector $ x $ and a scalar $1$. Since we do not treat
the bias variable $ b $ differently from the coefficients $ \omega $,
we view all the inner weights in one node as a $ d+1 $ dimensional vector
and call it $ \omega $. Hence, the function expressed by $ N $ random features
will appear as
\begin{equation*}
  f_N(x) = \sum_{i=1}^N c_i\sigma(\omega_i\cdot(x,1))\,.
\end{equation*}
We are most interested in the case where $ \sigma(z) =\max(0,z) $, which is
the ReLU node.
% A function of homogeneity degree $k$ on $ \mathbb{R} $ satisfies
% $ \sigma(ax)=a^k\sigma(x) $ for any $ a>0 $.
% It is obvious that such a function
% is fully determined by its value at $ \pm 1 $ as follows,
% \begin{equation*}
%   \sigma(x) =
%   \begin{cases}
%     a_{-1}x^k & \quad x\le 0~, \\
%     a_{1}x^k & \quad x > 0~. \\
%   \end{cases}
% \end{equation*}
% By setting $k=1$, it is clear that these functions include ReLU and leaky ReLU activation nodes.
% For simplicity, we just call them ReLU nodes in the rest of the paper.
% We denote $ \max\{\vert a_1\vert,\vert a_{-1}\vert\} $ by $ a_{\mathrm{ReLU}} $ in the paper.

Other notations include: $ \mathbb{P} $ denotes the data distribution over
$ \mathcal{X} $; $ |\cdot| $ means the Euclidean norm when the operand
is a vector in $ \mathbb{R}^d $, but the total variation norm when the
operand is a measure; $ f_{i:j} $ represents the composition of functions
$f_j\circ\cdots\circ f_i $; $\mathbf{R}^\ell$ denotes the expected risk with
respect to the loss $ \ell $; and $ \tau_d $ denotes the uniform probability
distribution on $ \mathbb{S}^d $.

%% file: universality-relu-feature.tex
%!TEX root = main.tex
\section{Universality of Random Features}
\label{sec:universality}
We start with the universality of the RKHS induced by random features of
neural network type.
\begin{theorem}
  \label{thm:universality}
  Assume that $ |\sigma(z)| \le K|z|^k + M $
  for some $ M,K\ge 0 $ and $ k\in \mathbb{N}_0 $. Let
  $ \mu $ be a probability distribution whose support is dense in
  $ \mathbb{R}^{d+1} $ with
  $ \int\vert\omega\vert^{2k}~\mathrm{d}\mu(\omega) \le M_2 $.
        If $ \sigma $ is not a polynomial, the RKHS $ \mathcal{H}_{\sigma,\mu} $ is universal.
\end{theorem}

Theorem~\ref{thm:universality} shows that
the RKHSs induced by a broad class of random features of neural network type
are dense in the space of continuous functions. And hence the random features
satisfying these conditions have a strong approximation capability.
% Its proof is a combination of Lemma~\ref{lem:universal}
% stated in Appendix~\ref{app:universality} and \cite{Leshno1993}'s results on the
% universality of neural networks.

Theorem~\ref{thm:universality} requires the feature distribution to be supported almost
everywhere on $ \mathbb{R}^{d+1} $. However, when the feature map is ReLU,
this requirement can be relaxed due to the homogeneity.
The following proposition provides us a sufficient condition for the
universality of $ \mathcal{H}_{\mathrm{ReLU},\mu} $.
\begin{proposition}
  \label{prop:universality-relu}
  When $ \mu $ is a probability distribution supported on a dense subset of a $d$-dimensional ellipsoid centered at the origin, the RKHS
        $\mathcal{H}_{\mathrm{ReLU},\mu}$ is universal.
\end{proposition}

This proposition is only stated for the regular feature space like
$ \mathbb{S}^d $ or ellipsoids around the origin for
simplicity of the statement.
The proof actually works for any set $ \Omega $ satisfying the following
property: for any $ \omega\ne 0 $ in $ \mathbb{R}^{d+1} $, there exists
a scalar $ c > 0 $ such that $ c\omega\in \Omega$.
This proposition allows us to restrict the
ReLU feature map to be a bounded function when the data set is bounded.
And this will be important for us when we prove the approximation capability of
random ReLU features.

As we mentioned in the introduction, the universality of
$ \mathcal{H}_{\mathrm{ReLU},\mu} $, when $ \mu $ is the
uniform distribution over the unit sphere, has been shown implicitly by
\cite{Bach2017}. However, due to the non-constructive approach we take,
it is much easier for us to prove the universality of the RKHS induced by
$(\mathrm{ReLU},\mu)$ with $\mu$ supported over various domains.

In the next, we study the approximation properties of random ReLU features.
As we mentioned in the introduction, since random features methods are not
deterministic, the approximation property has to be stated with probability.
Therefore we first give the following definition.
\begin{definition}
  \label{def:universality}
  Given a random feature $ (\sigma, \mu) $, if for any $ f $ in
  $ C(\mathcal{X}) $ and $ \delta,\epsilon>0 $, there exist a positive integer
  $ N $ such that with probability greater than $ 1-\delta $, we can find coefficients $ \{c_i\}_{i=1}^N $ such that
  \[ \Vert f_N - f\Vert_{L^2(\mathbb{P})}
  < \epsilon\,,\]
  we say that the random feature is universal.
\end{definition}

If we know that the RKHS induced by the random feature
$ (\sigma,\mu) $ is universal, to show that it is universal
in the sense of Definition~\ref{def:universality},
we need only verify that any function
$ f \in \mathcal{H}_{\sigma,\mu} $ can be approximated
by a linear combination of finitely many random features, $ f_N $,
with high probability of random choice of $ \omega_i $.
This has been studied in \cite{Bach2017a}.
To apply Bach's result, we make the following definition.
\begin{definition}
  A random feature $ (\sigma,\mu) $ is called admissible
  if for any $\lambda>0$,
  \[ \sup_{\omega\in\Omega}\langle \sigma_\omega,
  (\Sigma+\lambda I)^{-1}\sigma_\omega \rangle < \infty\,,\]
  where $ \Sigma:L_2(\mathbb{P})\to L_2(\mathbb{P}) $ is defined by
  $ \Sigma f = \int k_{\sigma,\mu}(x,y)f(y)~\mathrm{d}\mathbb{P}(y)$,
  and $\sigma_\omega(x) = \sigma(\omega\cdot(x,1))$.
\end{definition}

This property is used in earlier work \citep{Sun2018b,Rudi2017} to efficiently obtain random features
in their learning rate analysis.
It is a data-dependent property except for special cases such as bounded
feature maps.
How to design the admissible random features remains an open problem.
Though it is data-dependent, it does not depend on the labels or target
functions. And hence we may gain advantages in supervised learning tasks if we
can design random features according to the characteristics of the data
distribution.

With this admissibility assumption, we have the following sufficient condition
for the universality of random features. This answers the question we
asked in the introduction.
\begin{corollary}
  \label{cor:opt-universal}
  Assume that $ |\sigma(z)| \le K|z|^k + M $
  for some $ M,K\ge 0 $ and $ k\in \mathbb{N}_0 $. Let
  $ \mu $ be a probability distribution whose support is dense in
  $ \mathbb{R}^{d+1} $ with
  $ \int\vert\omega\vert^{2k}~\mathrm{d}\mu(\omega) \le M_2 $.
  If $ \sigma $ is not a polynomial and $ (\sigma,\mu) $ is admissible with respect
  to the data distribution $ \mathbb{P} $, the random feature
  $ (\sigma,\mu) $ is universal.
\end{corollary}
Though the admissibility assumption is usually data-dependent, the random feature $(\sigma,\mu)$
with a bounded $\sigma$ is always admissible. Therefore, we have the following corollary.

\begin{corollary}
  \label{cor:bounded}
  When $ \sigma $ is bounded and $ \mathcal{H}_{\sigma,\mu} $ is
  universal. The random feature $ (\sigma,\mu) $ is universal.
\end{corollary}

\cite{Huang2006a} proved that under the assumptions of
Corollary~\ref{cor:bounded}, for any continuous function $ f $ and randomly
generated $ \omega_i $'s, with probability $ 1 $, there exist $ c_i $'s such that $ f_N $ converges
to $ f $ under $ L^2(\mathbb{P}) $-norm as $ N $ goes to infinity. Note that Definition~\ref{def:universality} only
requires the convergence in probability instead of almost surely, and thus
Corollary~\ref{cor:bounded} seems to be weaker than Huang's result.
They are actually equivalent. To see this, first note that
the statement is on the existence of
the approximator in the linear space of $ n $ random basis elements.
Denote by $ E(\{\omega_i\}_{i=1}^n) $ the space spanned by
$ \{\sigma(\omega_i\cdot(x,1))\}_{i=1}^n $.
The convergence in probability stated in Corollary~\ref{cor:bounded}
implies that there exists a subsequence $ \{n_k\}_{k=1}^\infty $ such that
with probability 1, the infinite sequence of $ \{\omega_i\}_{i=1}^\infty $
sampled randomly satisfies that $ E(\{\omega_i\}_{i=1}^{n_k}) $
contains an approximator converging to the target function
as $ k $ goes to infinity. Since $ E(\{\omega_i\}_{i=1}^n)\subset
E(\{\omega_i\}_{i=1}^m) $ whenever $ n\le m $, the almost sure convergence
holds for all $ n $.

Corollary~\ref{cor:bounded} shows that random ReLU features with feature distribution
$ \mu $ supported on ellipsoids around the origin can approximate any continuous
function with high probability over the random choice of
parameter $ \omega_i $s. The proofs of all the results in this section
can be found in Appendix~\ref{app:universality}.
We will give a detailed description of the random ReLU features method in the
supervised learning context and discuss its performance in
Section~\ref{sec:learn-rate} and Section~\ref{sec:experiments}.

%% file: gen-error.tex
%!TEX root = main.tex

\section{Universal Consistency of the Random ReLU Features Method}
\label{sec:learn-rate}

Algorithm~\ref{alg:ran-relu} describes the supervised learning algorithm using
random ReLU features.

\begin{algorithm}
  \SetKwInOut{Input}{input}
  \SetKwInOut{Output}{Output}
  \Input{$\{(x_i,y_i)\}_{i=1}^m,\gamma,R,N$}
  \Output{$ f_N(x)=\sum_{j=1}^N c_j\mathrm{ReLU}(\omega_j \cdot (x_i,1/\gamma)) $}

  Generate $\{\omega_j\}_{j=1}^N\subset\mathbb{S}^d$ according to $ \tau_d $, the uniform
  distribution over $\mathbb{S}^d$\;
  Choose appropriate loss function $\ell$ according to the type of tasks, and
  solve the optimization problem:
  \begin{equation*}
      \operatornamewithlimits{\mathrm{minimize}}_{\sum_{j=1}^N c_j^2 \le R^2}~
      \frac1m\sum_{i=1}^m\ell\left(\sum_{j=1}^N c_j\mathrm{ReLU}(\omega_j
      \cdot (x_i,1/\gamma)), y_i\right)
  \end{equation*}
  \caption{Random ReLU features method.}
  \label{alg:ran-relu}
\end{algorithm}

The first hyper-parameter $ \gamma $ plays a similar role of the bandwidth parameter in random Fourier
features method. Since the numerical ranges of features in different datasets may be large, introducing the bandwidth parameter can help normalize the
data.

The second hyper-parameter $ R $ is the constraint on the 2-norm of the outer weights during the training.
Considering the constrained instead of regularized form of optimization simplifies the generalization error
analysis, and it is also practical based on our experiment results.
In fact, since for a fixed number of random features, the capacity of the hypothesis class is limited,
random features methods are less vulnerable to overfitting than traditional kernel method.
This has been confirmed that carefully choosing the batch size and step size can avoid overfitting
without use of regularizer or norm constraint in random features methods in regression tasks
\citep{Carratino2018}.

% In the {\em realizable} case, we have the following bound on the sample complexity.
% \begin{proposition}
%   \label{thm:gen-error}
%   Let $\ell$ be the hinge or logistic loss. Assume that the target function $f_0\in\mathcal{H}_{\mathrm{ReLU},\tau_d}$ with RKHS norm less than $R_{f_0}$.
%   For any $\delta,\epsilon>0$, if we have $ m = O(R\epsilon^{-2}\log(\delta^{-1}))$ samples
%   and $ N=O(\epsilon^{-2}\log((\epsilon\delta))^{-1}$ features.
% %   \[
% %     m\ge \left[\left(4+2\sqrt{2\ln\frac1\delta}\right)\frac{R(\sqrt{r^2+1}
% %     + 1)}{\epsilon}\right]^2\,,
% %   \]
% %   and
% %   \[
% %     N\ge \frac{5(r^2+1)}{\epsilon^2}\ln\left(
% %     \frac{16(r^2+1)}{\epsilon^2\delta}\right)\,.
% %   \]
%   with probability greater than $ 1-2\delta $, we have
%   \begin{equation*}
%   \mathbf{R}^{\ell}(f_N) - \mathbf{R}^{\ell}(f^*) \le 3\epsilon\,,
%   \end{equation*}
% \end{proposition}

% This bound guarantees that when the number of samples and the number of
% features are large enough, the learning algorithm described above
% will return a solution whose performance is no worse than the true regression function
% or classifier in the RKHS by $ \epsilon $.
Now we show that the random ReLU features method is a universally consistent
supervised learning algorithm.
For any function $f_0\in\mathcal{H}_{\mathrm{ReLU},\tau_d}$, the standard
statistical learning theory guarantees that Algorithm~\ref{alg:ran-relu} will return a
solution with excess risk sufficiently small if the sample size $m$ and the number of features
$N$ are large enough and $R$ is chosen to be greater than $2\Vert f_0\Vert_{\mathrm{ReLU},\tau_d}$;
see Appendix~\ref{app:learn-rate} for details.
Meanwhile, by Proposition~\ref{prop:universality-relu}, we know that for any continuous function $ f^* $
and $ \epsilon > 0 $, there exists $ f_0\in\mathcal{H}_{\mathrm{ReLU},\tau_d} $
such that
\begin{equation*}
\Vert f_0-f^*\Vert_{L^2(\mathbb{P})} \le \epsilon\,.
\end{equation*}
This proves the following statement.
\begin{proposition}
The random ReLU features method (Algorithm~\ref{alg:ran-relu}) is
universally consistent.
\end{proposition}

To obtain a meaningful learning rate of the random ReLU features method beyond the universal
consistency of the algorithm, we need to obtain tight upper bounds on
the generalization and approximation errors. The generalization
error bound tighter than Proposition~\ref{prop:gen-error}
in Appendix~\ref{app:learn-rate} can be obtained
by using local Rademacher complexity as in \cite{Sun2018b}. However,
it is not clear yet how to obtain a tight bound on the approximation error.
This is also discussed in Appendix~\ref{app:learn-rate}.
We compare the performance of the random ReLU features method on several
real and synthetic datasets with the random Fourier features method in
Section~\ref{sec:experiments}.

%% file: barron.tex
%!TEX root = main.tex
\section{Composition of Functions in the RKHS Induced by Random ReLU Features}
\label{sec:barron}

Random features models belong to so-called shallow models.
It is widely believed that an advantage of deep models over shallow
ones is their capability to efficiently approximate composite functions.
The following proposition shows that the composition of functions in
$\mathcal{H}_{\mathrm{ReLU},\tau_d}$ generates substantially more complicated functions
that are hard to be approximated by functions in $\mathcal{H}_{\mathrm{ReLU},\tau_d}$.
\begin{proposition}
  \label{prop:separation}
  There exist universal constants $ c,C $ such that for any $ d>C $,
  we can construct a probability measure $ \mathbb{P} $ supported
  on a compact set
  $ \mathcal{X}\subset \mathbb{R}^{d} $, and two functions $ f:\mathcal{X}\to\mathbb{R} $ and
  $ g:\mathbb{R}\to\mathbb{R} $
  with $ \Vert f\Vert_{\mathrm{ReLU},\tau_{d}} $ and $ \Vert g\Vert_{\mathrm{ReLU},\tau_1} $
  less than $ \mathrm{poly}(d) $ such that the following holds. For every function $ h:\mathbb{R}^d\to\mathbb{R} $ in $ \mathcal{H}_{\mathrm{ReLU},\tau_d} $
  with $ \Vert h\Vert_{\mathrm{ReLU},\tau_d} \le C\exp(cd)$,
  \begin{equation*}
    \Vert h-g\circ f\Vert_{L^2(\mathbb{P})} \ge c .
  \end{equation*}
\end{proposition}

The proof is based on the construction of \cite{Daniely2017}. We
first check that the functions considered in his work
belong to $\mathcal{H}_{\mathrm{ReLU},\tau_d}$ and $\mathcal{H}_{\mathrm{ReLU}
,\tau_1}$ with $ \mathrm{poly}(d) $ norm. The key lemma in this step,
Lemma~\ref{lem:proj_fn_norm}, shows that power functions of the projections
over a given direction belong to $\mathcal{H}_{\mathrm{ReLU},\tau_d}$ with
a small norm.
Then by Daniely's work, we know that such a composite function
can not be approximated
by any 2-layer ReLU networks with all the weights and the number of nodes
less than $ O(\exp(d)) $. On the other hand, we show in Corollary~\ref{cor:main}
that any functions in the RKHS with norm less than $O(\exp(d))$ can be
approximated by 2-layer ReLU networks with weights and number of nodes
less than $O(\exp(d))$. This completes the proof of
Proposition~\ref{prop:separation}. See Appendix~\ref{app:multi}
for details.

Let's define the RKHS norm for a vector valued function as the $2$-norm of the RKHS norm of its components;
that is, for $ f:\mathbb{R}^n\to\mathbb{R}^m $ in $\mathcal{H}^{\oplus m}$ ,
$ \Vert f\Vert^2_\mathcal{\mathcal{H}} \coloneqq \sum_{i=1}^m \Vert (f)_i\Vert_\mathcal{H}^2 $.
Proposition~\ref{prop:multi} shows that the composition of functions in $ \mathcal{H}_{\mathrm{ReLU},\tau} $
can always be efficiently approximated by a multilayer ReLU networks.

\begin{proposition}
  \label{prop:multi}
  Assume that for all $ 1\le i\le L+1 $, $ K_i $ is a compact set with
  radius $ r $ in $ \mathbb{R}^{m_i} $, among which $ K_1=\mathcal{X} $.
  Let $ B^{m_i} $ denote the unit ball in $ \mathbb{R}^{m_i} $.
  $ \mathbb{P} $ is a probability measure on $ K_1 $.
  $ f_i:K_{i}+sB^{m_{i}}\to K_{i+1} $ belongs to
  $ \mathcal{H}_{\mathrm{ReLU},\tau_{m_i}}^{\oplus m_{i+1}} $ with RKHS norm less than $ R_i $, for an $ s>0 $ and any
  $ 1\le i\le L $. Then for
  $ \epsilon>0 $, there exists an $ L $-layer neural networks $ g_{1:L} $ with $ m_{i+1}N_i $ nodes in each layer, where
  \[
    N_i = \frac{\prod_{j=i}^LR_j^2((r+s)^2+1)}{\epsilon^2}\,,
  \]
  such that
    \[\left(\int_{\mathcal{X}}|f_{1:L}-g_{1:L}|^2~\mathrm{d}\mathbb{P}\right)^{1/2}
    \le C\epsilon\,,\]
    where $C$ is a constant depending on $L,s,r,\{R_i\}_{i=1}^L$.
  Moreover, the Frobenius norm of all the weight matrices are bounded by constants depending on $ R_i,r,s,m_i $ and bias
  terms are bounded by $1$.
\end{proposition}

%% file: experiments.tex
\section{Experiments}
\label{sec:experiments}

We first compared the performance of the random ReLU features with the popular
random Fourier features with the Gaussian feature distribution on four synthetic
datasets (see Appendix~\ref{app:graphs}) and three real datasets:
MNIST (\cite{Lecun}), adult and covtype (\cite{Dheeru2017}).

For all four synthetic datasets,
we used 20 random features for each method; for real datasets we used
2000 random features. We used hinge loss in binary classification tasks, and logistic loss in multi-class classifiction tasks.
We chose to constrain the 2-norm of the outer weights by a large constant ($10^3$ for
synthetic datasets and $10^4$ for real datasets) as described in
Section~\ref{sec:learn-rate}.
The optimization method was the plain stochastic
gradient descent and the model was implemented using
Tensorflow \citep{MartinAbadi2015}. The learning rate and bandwidth parameters were
screened carefully for both models through grid search.

In Figure~\ref{fig:screen}, we present the dependence of two methods
on the bandwidth parameters in the screening step.
Each point displays the best 5-fold cross validation accuracy among all
learning rates. We can see that the performance
of the random Fourier features with Gaussian distribution is more sensitive to
the choice of bandwidth than the random ReLU features method.
\begin{figure}[!htb]
        \centering
        \includegraphics[width=0.33\textwidth]{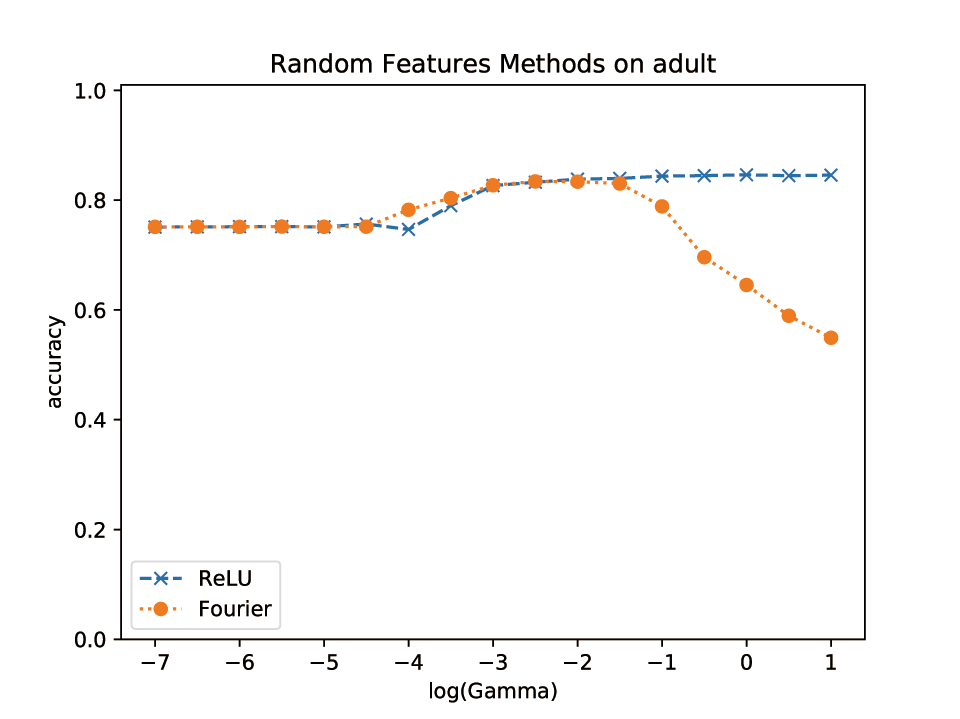}%
        \hfill%
        \includegraphics[width=0.33\textwidth]{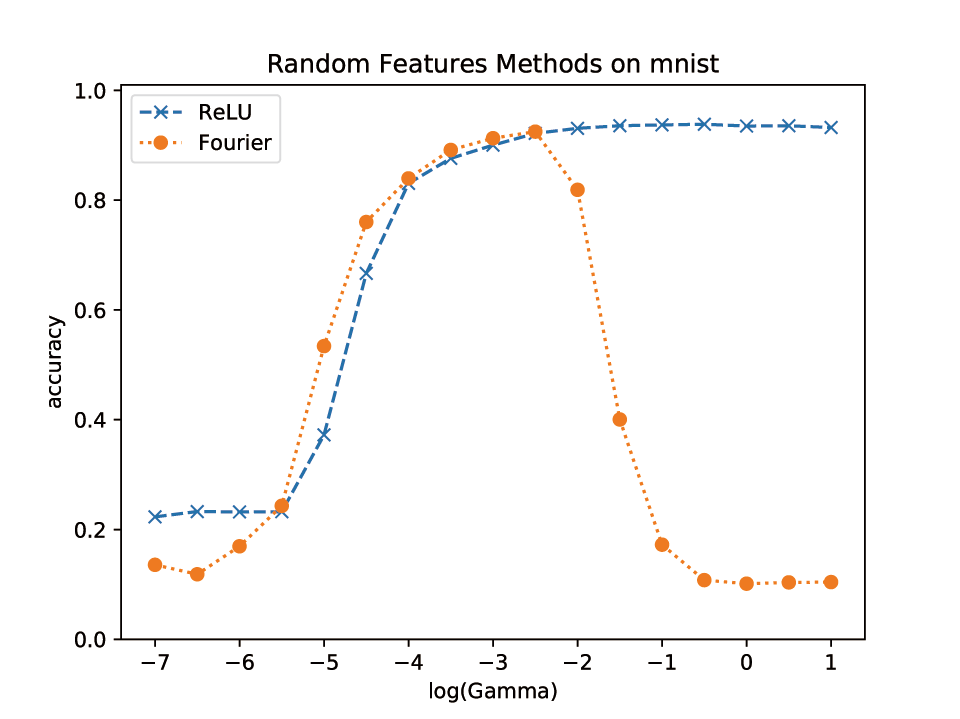}
        \hfill%
        \includegraphics[width=0.33\textwidth]{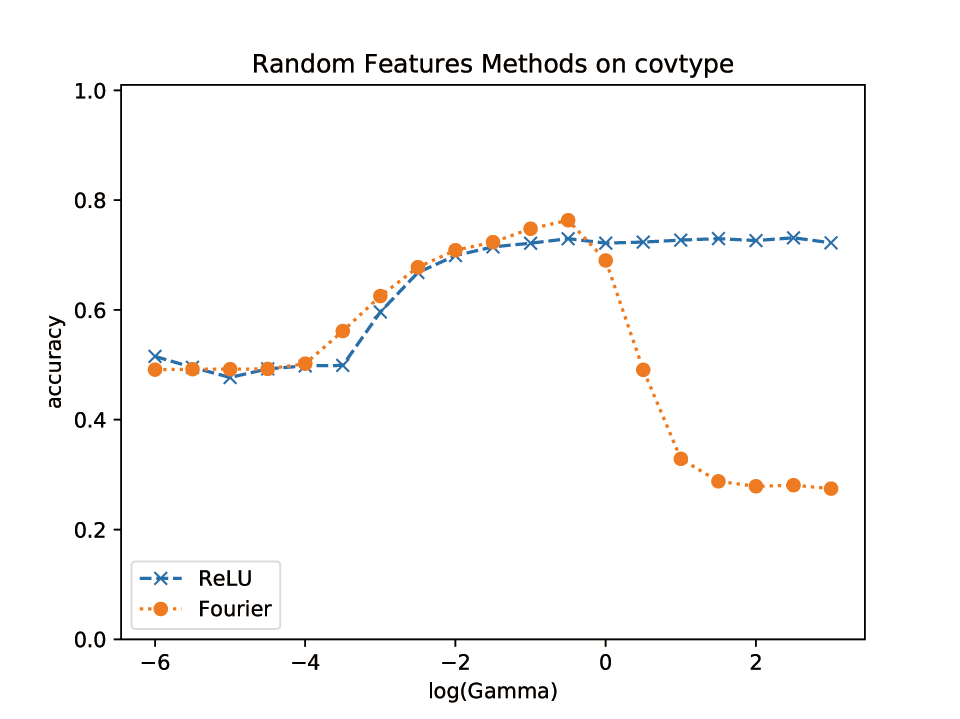}
    \caption{Cross validation accuracy of random Fourier features and
    random ReLU features. Left: adult. Middle: mnist. Right: covtype.}
    \label{fig:screen}
\end{figure}
We list the accuracy and training time for two methods in
Table~\ref{tab:experiments}.
For all the datasets, the random ReLU features method requires
shorter training time. By checking the feature vectors, we found that
half of coordinates of the random ReLU feature vectors are zeros, while none of coordinates of
the random Fourier feature vectors are zeros. The random ReLU features method outperforms the random Fourier
features with higher accuracy on adult and MNIST datasets.
Its performance is similar to the random Fourier features method on sine, checkboard
and square datasets. However, its performance on strips and covtype dataset
is significantly worse.
\begin{table}[htb]
  \caption{Left: accuracy of random ReLU features versus random Fourier features.
        Right: training time of random ReLU features versus random Fourier features;
        The unit is second. The results are averages over 10 trials.}
  \label{tab:experiments}
  \begin{tabular}{lcc}
    \toprule
    & Fourier & ReLU \\
    \midrule
    sine & 0.993(0.007) & 0.984(0.005) \\
    strips & 0.834(0.084) & 0.732(0.006) \\
    square & 0.948(0.038) & 0.934(0.015) \\
    checkboard & 0.716(0.045) & 0.743(0.027) \\
    adult & 0.838(0.002) & 0.846(0.002) \\
    mnist & 0.937(0.001) & 0.951(0.001) \\
    covtype & 0.816(0.001) & 0.769(0.002) \\
    \bottomrule
  \end{tabular}%
  \hfill%
  \begin{tabular}{lcc}
    \toprule
    & Fourier & ReLU \\
    \midrule
    sine & 1.597(0.050) & 1.564(0.052) \\
    strips & 1.598(0.056) & 1.565(0.052) \\
    square & 1.769(0.061) & 1.743(0.057) \\
    checkboard & 1.581(0.078) & 1.545(0.073) \\
    adult & 6.648(0.181) & 5.849(0.216) \\
    mnist & 70.438(0.321) & 69.229(1.080) \\
    covtype & 125.719(0.356) & 112.613(1.558) \\
    \bottomrule
  \end{tabular}
\end{table}

The depth separation and multi-layer approximation results in
Section~\ref{sec:barron} only prove the {\em existence} of the
advantage of deeper models. It is not clear whether we can {\em find}
good multi-layer approx\-imators with significantly better performance
than shallow models.
Therefore, we designed a synthetic dataset that is supposedly hard to learn by shallow models
according to the depth separation result, and used it to show the performance gap between deep and shallow models.
The data uniformly distribute over $\mathbb{S}^1\times\mathbb{S}^1$ and
the labels are generated by the target function $\sin(8\pi(x_1x_3+x_2x_4))$;
see Figure~\ref{fig:daniely-data}.
For this regression task, we trained three types of models:
random ReLU features models, 2-layer neural networks, and 3-layer neural networks.
The random ReLU features models and 2-layer neural networks have exactly the same structure.
The only difference is whether the inner weights are randomly chosen or trained together with the outer weights.
The number of nodes in the hidden layer of the random ReLU features models
and 2-layer neural networks ranges from 20 to 5120 with each level doubling the
preceding one.
The 3-layer neural networks are fully connected with the equal width in each layer.
To make a fair comparison, we fixed the total number of
parameters of the 3-layer neural networks
to be equal to the shallow models at each level.
We used the adam optimizer in Tensorflow for the training of three types of
models with learning rates screened using holdout validation.
\begin{figure}[!htb]
        \centering
        \includegraphics[width=0.33\textwidth]{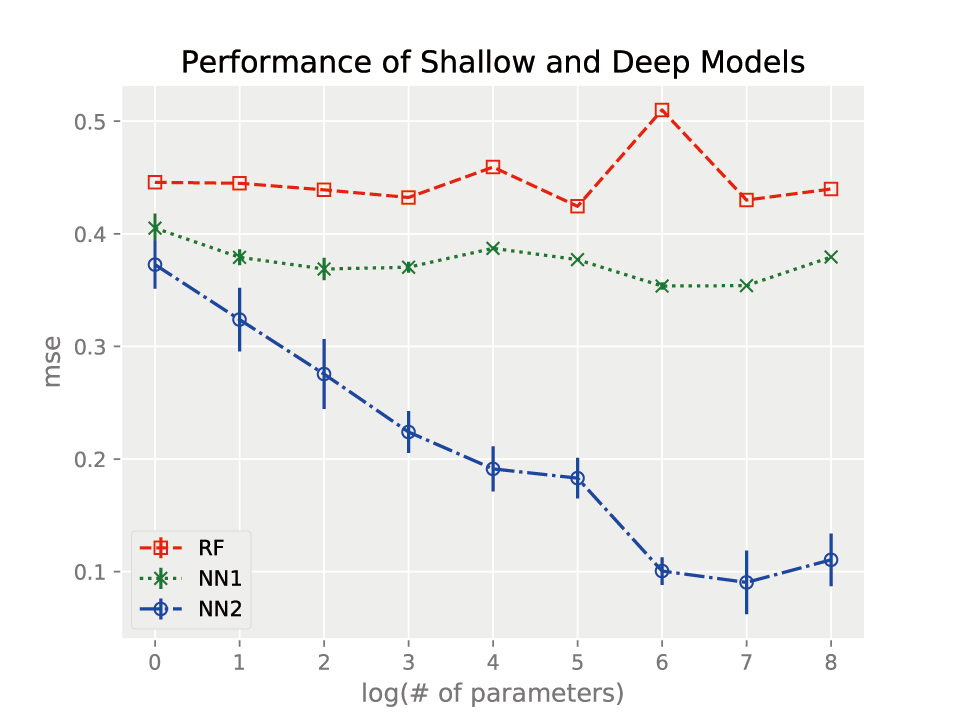}%
        \hfill%
        \includegraphics[width=0.33\textwidth]{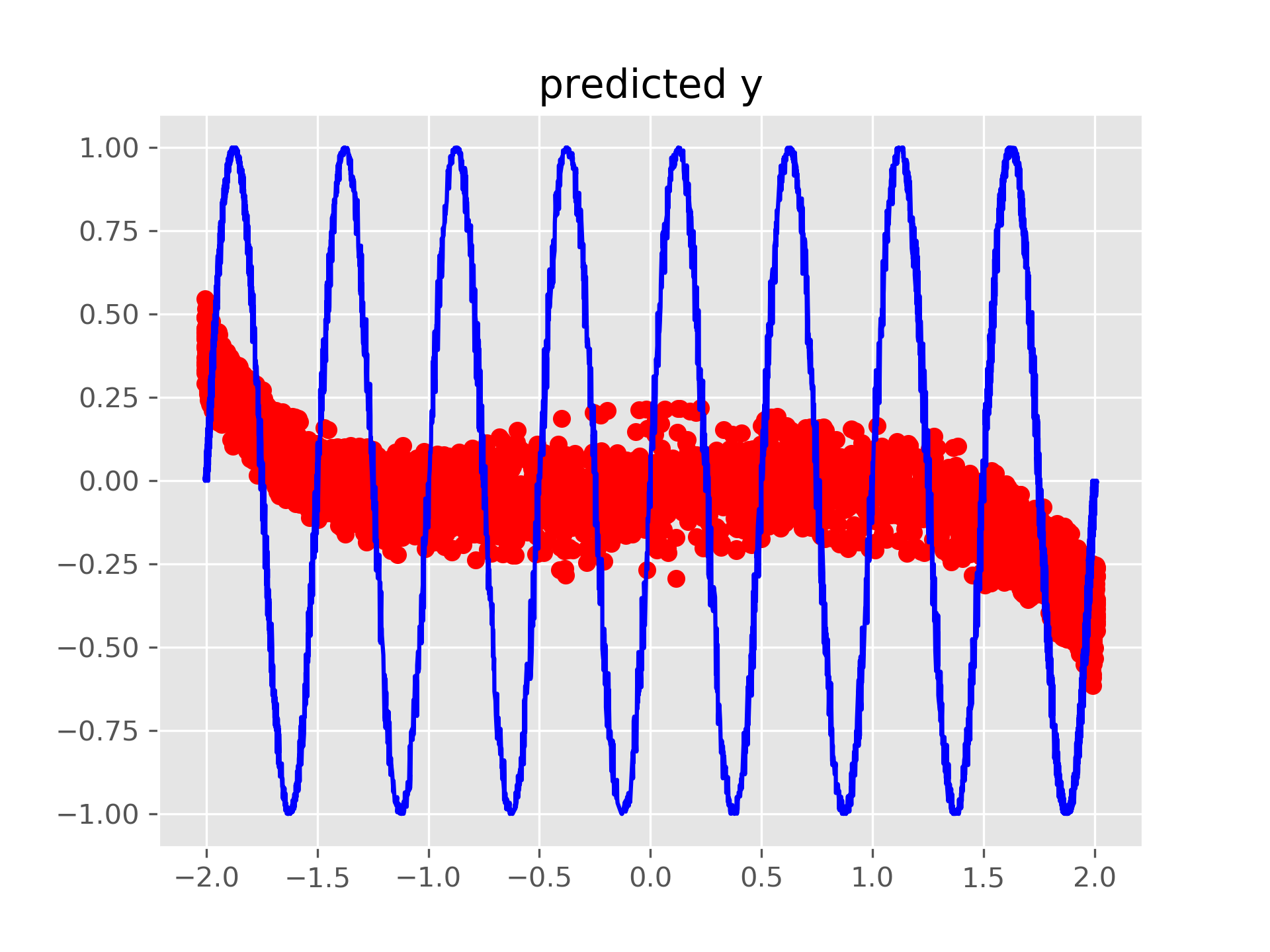}
        \hfill%
        \includegraphics[width=0.33\textwidth]{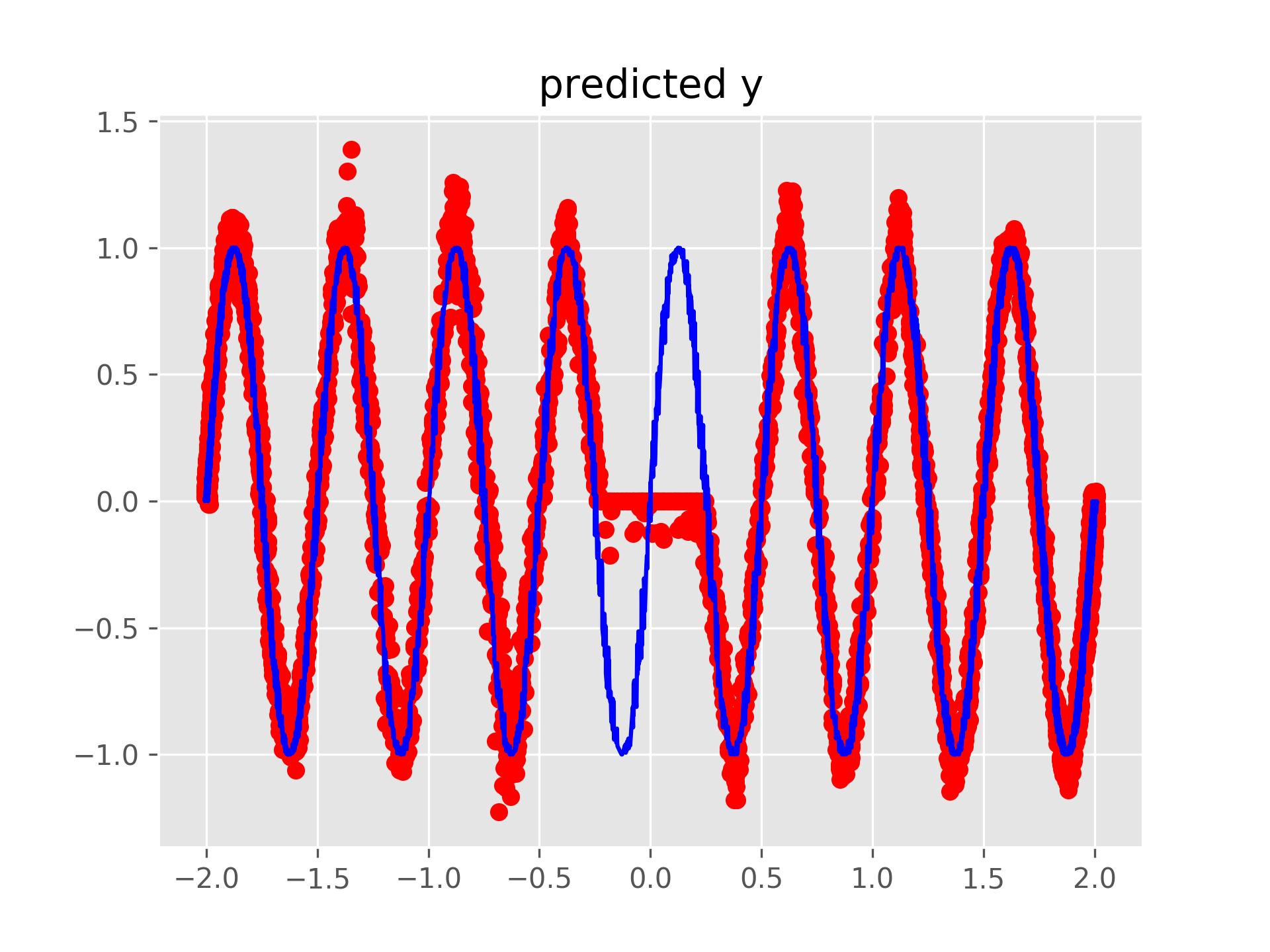}
    \caption{Performance of the deep and shallow models in the regression task. Left: mean squared error of three models at different levels of total parameters. The results are averages over 10 trials. Middle: the predicted labels (red) by the best random ReLU features model compared to the true labels (blue) plotted against the dot product $x_1x_3+x_2x_4$. Right: the predicted labels (red) by the best 3-layer neural networks compared to the true labels (blue).}
    \label{fig:depth}
\end{figure}
From Figure~\ref{fig:depth}, we can see that 3-layer neural networks consistently
achieve significantly better performance whereas the gap between
the 2-layer neural networks and random ReLU is not very large.
By plotting the predicted labels on test dataset, we can see that the 3-layer neural network with the best performance indeed learns a
function closer to the target function. In contrast, the random ReLU
features model learns a more regular function and does not fully adapt to
the rapid oscillation of the target function. Considering that 2-layer neural networks have much more
adjustable parameters than random ReLU models, it is surprising
that the 2-layer neural networks' performance is not significantly
better than the random ReLU.
This implies that the structure of models has a more important
impact on the performance than the number of adjustable parameters.
We also run experiments based on Eldan and Shamir's construction and observe
similar phenomenon; see Appendix~\ref{app:graphs}.
The code of all the experiments can be found at:
% [URL hidden for blind review]
\url{https://github.com/syitong/randrelu}

%% file: conclusion.tex
%!TEX root = main.tex
\section{Conclusion}

The study of universality of random features provides
theoretical foundation of designing new random features
algorithms. By comparing with random Fourier features, we believe that random
ReLU features can bring improvement in computational cost in many tasks.
The depth separation result shows the drawbacks of shallow models.
But designing random features with composition structure is worth studying further.
The performance comparison between random ReLU features models and
2-layer ReLU networks shows the possible trade-off
between the number of random features
and the number of total trainable parameters, which is interesting
from the theoretical viewpoint since the optimality is
always achievable for random features methods.

%% file: pf_universality.tex
%!TEX root = main.tex
\section{Proof of Universality of Random Features}
\label{app:universality}

To show that a subset is dense in a Banach space,
we need only consider its annihilator as described by the following lemma.
\begin{lemma}
  \label{lem:annihilator}
  For a Banach space $ \mathcal{B} $ and its subset $ U $, the linear
  span of $ U $ is dense in $ \mathcal{B} $ if and only if $ U^{\bot} $,
  the annihilator of $ U $, is $ \{ 0 \} $.
\end{lemma}
The proof can be easily derived from Theorem~8 in Chapter~8 of \cite{Lax2002}.
It is a consequence of Hahn-Banach theorem.
The dual space of $ C(\mathcal{X}) $ is the space of all signed measures
equipped with the total variation norm, denoted by $ M(\mathcal{X}) $
(see Theorem~14 in Chapter~8 of \cite{Lax2002}). As the consequence of
Lemma~\ref{lem:annihilator} and the duality between $ C(\mathcal{X}) $ and
$ M(\mathcal{X}) $, \cite{Micchelli2006} use the following useful criteria
for justifying the dense subset of $ C(\mathcal{X}) $.
\begin{lemma}
  \label{lem:universal}
  $ \mathcal{F}\subset C(\mathcal{X}) $ is universal
  if and only if for any signed measure $ \nu $,
  \begin{equation*}
    \int_{\mathcal{X}} f(x)~\mathrm{d}\nu(x) = 0 \quad\forall f\in \mathcal{F}
    \implies \nu=0\,.
  \end{equation*}
\end{lemma}

When the kernel or respectively the activation function is continuous,
the universality of the RKHS or that of neural nets can be established
by applying Lemma~\ref{lem:universal}.

\cite{Bach2017a} proved the following useful theorem, which we will use to
verify the finite approximatability of RKHSs induced by various random features.
\begin{theorem}[Bach's approximation theorem]
  \label{thm:app-err}
  Let
  \begin{equation*}
    d_{\max}(1,\epsilon) := \sup_{\omega\in\mathbb{R}^{d+1}}\Vert (\Sigma + \epsilon I)^{-1/2}
  \phi(\cdot;\omega)\Vert_{L_2(\mathbb{P})}^2\,,
  \end{equation*}
  where $ \Sigma:L_2(\mathbb{P})\to L_2(\mathbb{P}) $ is defined by
  \[ \Sigma f = \int k_{\sigma,\mu}(x,y)f(y)~\mathrm{d}\mathbb{P}(y)\,.\]
  For $ \delta,\epsilon>0 $, when
  \begin{equation}
    N\ge 5d_{\max} (1,\epsilon)\log\left(\frac{16d_{\max} (1,\epsilon)}{\delta}\right)\,,
  \end{equation}
  with probability over $\phi_N$ greater than $1-\delta$,
  \begin{equation}
    \sup_{\Vert f\Vert_{\phi,\mu}\le 1}\inf_{|\beta|\le 2}
    \Vert f-\beta\cdot\phi_N(\cdot)\Vert_{L^2(\mathbb{P})} \le 2\sqrt{\epsilon}\,.
  \end{equation}
\end{theorem}

We first prove Theorem~\ref{thm:universality}, which is a very general
result on the universality of RKHS induced by random features. It is a
combination of several functional analysis results and the result of \cite{Leshno1993}.
\begin{theorem*}[\ref{thm:universality}]
  Assume that $ |\sigma(z)| \le K|z|^k + M $
  for some $ M,K\ge 0 $ and $ k\in \mathbb{N}_0 $. Let
  $ \mu $ be a probability distribution whose support is dense in
  $ \mathbb{R}^{d+1} $ with
  $ \int\vert\omega\vert^{2k}~\mathrm{d}\mu(\omega) \le M_2 $.
        If $ \sigma $ is not a polynomial, the RKHS $ \mathcal{H}_{\sigma,\mu} $ is universal.
\end{theorem*}
\begin{proof}
  First, it is easy to see that
  $ \sigma(\omega\cdot (x,1))\in L^2(\mathbb{R}^{d+1},\mu) $.
  Indeed, since $ |\sigma(z)| \le K|z|^k + M $, we have
  \begin{align*}
    \int \sigma^2(\omega\cdot
     (x,1))\,\mathrm{d}\mu(\omega)
    & \le 2\int K^2
    (\omega\cdot (x,1))^{2k}\,\mathrm{d}\mu(\omega)
     + 2M^2 \\
    & \le 2K^2 \int \vert \omega\vert^{2k}(\vert x\vert^2+1)^k
    \,\mathrm{d}\mu(\omega) + 2M^2\\
    & \le 2K^2(r^2+1)^kM_2 + 2M^2\,,
  \end{align*}
  where $r$ is the radius of $\mathcal{X}$.
  Next, we show that the functions in $\mathcal{H}_{\sigma,\mu}$ are all
  continuous. For $f\in\mathcal{H}_{\sigma,\mu}$, assume that
  \[f(x)=\int_\Omega\sigma(\omega\cdot(x,1))g(\omega)~\mathrm{d}\mu(\omega)\]
  for some $g\in L^2(\mu)$. To show that $f$ is continuous at a given point 
  $x$, we want to show that for any $\epsilon>0$, there exists $\delta>0$ 
  such that $|f(x)-f(y)|<\epsilon$ whenever $|x-y|<\delta$. Denote 
  \[I_1(R) \coloneqq \left\vert\int_{|\omega|>R} (\sigma(\omega\cdot(x,1))-
      \sigma(\omega\cdot(y,1)))g(\omega)~\mathrm{d}\mu(\omega)\right\vert\,,\]
  and
  \[I_2(R) \coloneqq \left\vert\int_{|\omega|\le R} (\sigma(\omega\cdot(x,1))-
      \sigma(\omega\cdot(y,1)))g(\omega)~\mathrm{d}\mu(\omega)\right\vert\,.\]
  Then since
  \begin{align*}
      I_1(R) & \le \int_{|\omega|>R} \left(K(|\omega\cdot(x,1)|^k+|\omega\cdot(y,1)|^k)
      +2M\right)|g(\omega)|~\mathrm{d}\mu(\omega) \\
      & \le 2\sqrt{2}\Vert g\Vert_{L^2} \left(\int_{|\omega|>R} K^2|\omega|^{2k}
      (r^2+1)^k+M^2~\mathrm{d}\mu(\omega)\right)\,,
  \end{align*}
  we know that $I_1(R)\to 0$ as $R\to\infty$. In particular, for sufficiently large
  $R$, we have $I_1(R)<\epsilon/2$.
  On the other hand, since $\sigma$ is continuous, there exists $\delta_1>0$ such that
  $|\omega\cdot(x,1)-\omega\cdot(y,1)|<\delta_1$ implies that
  \begin{align*}
      I_2(R) & \le \int_{|\omega|\le R}\frac{\epsilon}{2\Vert g\Vert_{L^2}}|g(\omega)|~\mathrm{d}\mu(\omega)\\
      & \le \frac{\epsilon}{2}\,.
  \end{align*}
  So if we set $\delta=\delta_1/R$, we will have $|f(x)-f(y)|<\epsilon$.
  
  We just proved that all the functions in the $\mathcal{H}_{\sigma,\mu}$ 
  are all continuous. So we can use
  Lemma~\ref{lem:universal}
  to justify the universality. For a signed measure $ \nu $ with finite total
  variation, assume that
  \begin{equation*}
    \int_{\mathcal{X}} \int_{\mathbb{R}^{d+1}} \sigma(\omega\cdot (x,1))
    g(\omega)~\mathrm{d}\mu(\omega) \mathrm{d}\nu(x) = 0\,,
  \end{equation*}
  for all $ g\in L^2(\mathbb{R}^{d+1},\mu) $. We want to show that $ \nu $ must
  be the $ 0 $ measure. Since the function $ \sigma(\omega\cdot
  (x,1))g(\omega) $
  is integrable over $ \mu\times\nu $, by Fubini's theorem we have
  \begin{equation*}
    \int_{\mathbb{R}^{d+1}}\left(\int_{\mathcal{X}}
    \sigma(\omega\cdot (x,1))~\mathrm{d}\nu(x)\right)
    g(\omega)~\mathrm{d}\mu(\omega) \,,
  \end{equation*}
  equals $ 0 $ for all $ g\in L^2(\mathbb{R}^{d+1},\mu) $. Then
  \begin{equation}
    \label{eq:universality}
    \int_{\mathbb{R}^{d+1}} \sigma(\omega\cdot (x,1))~\mathrm{d}\nu(x)
    = 0~ \mu \text{-a.e.}
  \end{equation}
  Indeed the function of $\omega$ defined on the left hand side of
  Equation~\ref{eq:universality} has to be 0 everywhere because of continuity.
  Since $ \sigma $ is not a polynomial, by \cite{Leshno1993}, we know that
  $ \nu $ must be a 0 measure. If it is not, then there exists
  $ f $ in $ C(\mathcal{X}) $ such that $\int f~\mathrm{d}\nu = \epsilon $
  where $ \epsilon\ge 0 $. Because the linear span of
  $ \sigma(\omega\cdot (x,1)) $ is dense in $ C(\mathcal{X}) $,
  there must exist $ c_i $s and $ \omega_i $s such that
  \begin{equation*}
    \int\sum_{i=1}^k c_i\sigma(\omega_i\cdot (x,1))~\mathrm{d}
    \nu(x) \ge \frac{\epsilon}{2}\,.
  \end{equation*}
  This contradicts Equation~\ref{eq:universality}.
\end{proof}

It is interesting to note that in the proof the universality of the RKHS
can be derived from that of corresponding neural networks.

When the activation node $ \sigma $ is homogeneous, we can modify the above
proof to obtain Proposition~\ref{prop:universality-relu}.
\begin{proposition*}[\ref{prop:universality-relu}]
  When $ \mu $ is a probability distribution supported on a dense subset of an $d$-dimensional ellipsoid centered at the origin. Then the RKHS
        $\mathcal{H}_{\mathrm{ReLU},\mu}$ is universal.
\end{proposition*}
\begin{proof}
  The proof is the same up to Eq.~\ref{eq:universality}. We claim that
  \[
    \int_{\mathcal{X}} \mathrm{ReLU}(\omega\cdot (x,1))~
    \mathrm{d}\nu(x) = 0
  \]
  for all $ \omega \in \mathbb{R}^{d+1} $. If not, there exists $ \beta\ne0 $
  such that
  \[
    \int_{\mathcal{X}} \mathrm{ReLU}(\beta\cdot (x,1))~
    \mathrm{d}\nu(x) = \epsilon > 0.\,.
  \]
  Then there exists a positive constant $ c $ such that $ c\beta $ belongs to
  the ellipsoid. By homogeneity, we have
  \[
    \int_{\mathcal{X}} \mathrm{ReLU}\left(c\beta\cdot (x,1)\right)~
    \mathrm{d}\nu(x) = c\epsilon > 0\,.
  \]
  The function
  \[
    g:\omega\mapsto\int_{\mathcal{X}} \mathrm{ReLU}(\omega\cdot
    (x,1))~ \mathrm{d}\nu(x)
  \]
  is continuous over the ellipsoidal feature space $ \Omega $, as a consequence of the 
  fact that ReLU is Lipschitz. And thus there exists $ \delta > 0 $ such
  that $ g(\omega) $ is greater than $ c\epsilon/2 $ for all
  $ \omega \in B^{d+1}(c\beta,\delta)\cap \Omega $. This
  contradicts Eq.~\ref{eq:universality} and the fact that the support of $ \mu $ is dense over $\Omega$.
  Then by the same argument at the end of the proof of
  Theorem~\ref{thm:universality}, we complete the proof of the proposition.
\end{proof}

\begin{corollary*}[\ref{cor:opt-universal}]
  Assume that $ |\sigma(z)| \le K|z|^k + M $
  for some $ M,K\ge 0 $ and $ k\in \mathbb{N}_0 $. Let
  $ \mu $ be a probability distribution whose support is dense in
  $ \mathbb{R}^{d+1} $ with
  $ \int\vert\omega\vert^{2k}~\mathrm{d}\mu(\omega) \le M_2 $.
  If $ \sigma $ is not a polynomial and $ (\sigma,\mu) $ is admissible with respect
  to the data distribution $ \mathbb{P} $, the random feature
  $ (\sigma,\mu) $ is universal.
\end{corollary*}
\begin{proof}
  The condition guarantees that $\mathcal{H}_{\sigma,\mu}$ is dense in $C(\mathcal{X})$.
  For any $f$ in $C(\mathcal{X})$ and $\epsilon>0$, there exists $\tilde{f}\in\mathcal{H}_{\sigma,\mu}$
  such that $ \sup_x|f(x)-\tilde{f}(x)| < \epsilon $, which implies that 
  $ \Vert f-\tilde{f} \Vert_{L_2(\mathbb{P})} < \epsilon $. Assume that 
  \[ \tilde{f}(x) = \int_\Omega \sigma(\omega\cdot (x,1))
  g(\omega)~\mathrm{d}\mu(\omega)\,,\]
  where $ \Vert g \Vert_{L_2(\mathbb{\mu})} = G $. Since $(\sigma,\mu)$ is admissible, there exists 
  $\lambda \le \epsilon^2/(4G^2)$ such that 
  \[ d_{\max}(1,\lambda)=\sup_{\omega\in\Omega}\langle \sigma_\omega,
  (\Sigma+\lambda I)^{-1}\sigma_\omega\rangle <\infty\,.\]
  So by Theorem~\ref{thm:app-err}, there exists $N\in\mathbb{N}$ such that
  \[\Vert \tilde{f} - f_N\Vert_{L^2(\mathbb{P})} \le 2G\sqrt{\lambda} \le \epsilon\,,\]
  with probability over the random $\omega_i\sim\mu$ greater than $1-\delta$.
\end{proof}
\begin{corollary*}[\ref{cor:bounded}]
  When $ \sigma $ is bounded nonconstant and the support of $ \mu $ is dense
  in $ \mathbb{R}^{d+1} $. The random feature $ (\sigma,\mu) $ is
  universal.
\end{corollary*}
\begin{proof}
  We need only check that $(\sigma,\mu)$ is admissible. Assume $|\sigma|\le \kappa$. 
  Then, for any $\lambda>0$,
  \begin{equation*}
      \sup_{\omega\in\Omega}\langle\sigma_\omega,(\Sigma+\lambda I)^{-1}\sigma_\omega\rangle
      \le \lambda^{-1}\Vert\sigma_\omega\Vert_{L^2(\mathbb{P})}^2
      \le \frac{\kappa^2}{\lambda}\,. 
  \end{equation*}
\end{proof}

%% file: pf_learn-rate.tex
%!TEX root = main.tex
\section{Generalization and Approximation Errors of Random ReLU Features}
\label{app:learn-rate}
For the generalization error, we have the following proposition.
Note that even though the proposition is stated for 1-Lipschtz losses, it works for 
any Lipschitz constant. And for the regression task, the squared loss is still 
Lipschitz since the output of all the functions in our setup is bounded by constants. 
\begin{proposition}
\label{prop:gen-error}
  Let $\ell$ be 1-Lipschtz loss. 
  And $f_0\in\mathcal{H}_{\mathrm{ReLU},\tau_d}$ with norm less than $R$. If we choose
  $m$ samples and $N$ random features, where
  \[
    m\ge \left[\left(4+2\sqrt{2\ln\frac1\delta}\right)\frac{R(\sqrt{r^2+1}
    + 1)}{\epsilon}\right]^2\,,
  \]
  and
  \[
    N\ge \frac{5(r^2+1)}{\epsilon^2}\ln\left(
    \frac{16(r^2+1)}{\epsilon^2\delta}\right)\,,
  \]
  and set $\gamma=1$ and $R=2R_{f_0}$ in Algorithm~\ref{alg:ran-relu}.
Then,
  with probability greater than $ 1-2\delta $, we have
  \begin{equation*}
  \mathbf{R}^{\ell}(f_N) - \mathbf{R}^{\ell}(f_0) \le 3\epsilon\,,
  \end{equation*}
  where $f_N$ is the solution returned by Algorithm~\ref{alg:ran-relu}.
\end{proposition}
\begin{proof}
Recall that the radius of $ \mathcal{X} $ is $ r $.
Since $ f_0\in \mathcal{H}_{\mathrm{ReLU},\tau_d} $ and its norm is less than $R_{f_0}$, 
by Theorem~\ref{thm:app-err}, there exists
\begin{equation*}
g:x \mapsto \sum_{i=1}^N c_i\mathrm{ReLU}(\omega_i\cdot (x,1))\,,
\end{equation*}
with $ \sum c_i^2 \le 4R_{f_0}^2 $, such that with probability greater than $ 1-\delta $
over the features $\{\omega_i\}_{i=1}^N$,
\begin{equation}
\Vert g-f_0\Vert_{L^2(\mathbb{P})} \le \epsilon\,,
\end{equation}
given that
\[
  N\ge \frac{20(r^2+1)}{\epsilon^2}\ln\left(\frac{64(r^2+1)}{\epsilon^2\delta}\right)\,.
\]

This implies that
\begin{equation}
\mathbf{R}^\ell(g)-\mathbf{R}^\ell(f_0) \le \epsilon\,.
\end{equation}

Next, we want to bound the excess risk between the solution $ f_N $ returned by the algorithm and $ g $.
Denote the RKHS induced by ReLU and the empirical distribution over $\{\omega_i\}_{i=1}^N$ 
by $\mathcal{H}_{N}$. Then both $f_N$ and $g$ belongs to $B(2R_{f_0})$, 
the ball centered at $0$ with radius $2R_{f_0}$ in $\mathcal{H}_{N}$.
By the statistical theory (\cite{Mohri2012}) we know that
\begin{align}
\mathbf{R}^\ell(f_N) - \mathbf{R}^\ell(g) & \le 2\sup_{f\in B(2R_{f_0})}
\left\vert\mathbb{E}\ell(f(x),y)
-\frac{1}{m}\sum_{i=1}^m \ell(f(x_i),y_i)\right\vert \\
& \le 2\mathfrak{R}_m(\ell\circ B(2R_{f_0})) + \sqrt{2\ln\frac{1}{\delta}}
\frac{2R_{f_0}(\sqrt{r^2+1}+1)}{\sqrt{m}} \\
& \le \frac{4 R_{f_0}\sqrt{r^2+1}}{\sqrt{m}}
+ \sqrt{2\ln\frac{1}{\delta}}
\frac{2R_{f_0}(\sqrt{r^2+1}+1)}{\sqrt{m}}\,, \\
\end{align}
with probability greater than $ 1-\delta $.
Set
\[
  m = \left(4R_{f_0} + 2R_{f_0}\sqrt{2\ln\frac1\delta}\right)\frac{(\sqrt{r^2+1}+1)}{\epsilon^2}\,.
\]
Taking the union bound over the two inequalities, the proof is complete.
\end{proof}

For a function $f$ that does not belong to the RKHS of uniform ReLU
features, we want to find a function $g$ in the RKHS to approximate
it. In particular, to study the learning rate of random ReLU features
support vector machines, we need to know how to construct an approximator
in the RKHS for the Bayes classifier $\mathrm{sign}(2\eta(x)-1)$.
Here we may adopt the assumption used in \cite{Steinwart2008,Sun2018b} that $\Delta x\ge\tau$
for any $x\in\mathcal{X}_{1}\cup\mathcal{X}_{-1}$ and $\eta(x)>\frac{1}{2}$
for all $x$. Again, we consider that $\mathcal{X}\subset\mathbb{S}^{d}$.
Then following the proof of Proposition 3 of \cite{Bach2017}, we
can construct the approximator of the Bayes classifier by 
\[
g(x)=\int_{\mathbb{S}^{d}}\mathrm{sign}(2\eta(y)-1)\frac{1-r^{2}}{(1+r^{2}-2r(x\cdot y))^{(d+1)/2}}\,d\sigma_{d}(y)\,.
\]
Then by choosing $1-r=C(d)\delta^{-2/(d+1)}$, we have $\Vert g\Vert_{RKHS}\le\delta$
and 
\[
\sup_{x\in\mathcal{X}}\left|\mathrm{sign}(2\eta(x)-1)-g(x)\right|\le C(d,\tau)\delta^{-2/(d+1)}\,.
\]
And we also have $|g(x)|\le1$ for all $x$. If we further have that
the spectrum of the kernel operator against the data distribution
decays in the rate $\lambda_{i}=O(i^{-d/2})$, then we can prove that
the learning rate of random ReLU with optimized feature sampling will
achieve $O(m^{-1/8})$ with $m$ samples and $\sqrt{m}$ random features.
If the spectrum of the kernel operator against the data distribution
decays exponentially, the learning rate can be largely improved to
$O(1/m)$ with $\log m$ features. However, it is not clear how to
characterize the data distribution satisfying this requirement on
$\mathbb{S}^{d}$, if such a data distribution exists. 

%% file: pf_sep_rkhs.tex
\section{Proof of the Depth Separation}

Before the proof of the depth separation, we first show how to upper
bound the RKHS norm $ \Vert\cdot\Vert_{\mathrm{ReLU},\tau_d} $ of the function
$ \sqrt{\Vert x\vert^2+1} $. To get the upper bound, we start with
the polynomials of projection $ f(\tilde{x}) = \alpha(\beta\cdot \tilde{x})^{p} $
defined on $ \mathbb{S}^d $, and its RKHS norm in the space induced by
$ \mathrm{ReLU}(\omega\cdot\tilde{x}) $ and $ \tau_d $, the uniform distribution over
$ \mathbb{S}^d $. We denote this RKHS by $ \tilde{\mathcal{H}} $.
And we use $ \tilde{x} $ to emphasize that these are points on the unit
sphere, while $ x $ represents a point in $ \mathbb{R}^d $.

A sufficient condition for the membership of
RKHS induced by the random feature $ (\mathrm{ReLU},\tau_d) $ is
provided in \cite{Bach2017a}. It shows that functions with bounded
$ \frac{d}{2} + \frac32 $ derivatives are in
$ \mathcal{H}_{\mathrm{ReLU},\tau_d} $, and the RKHS norm is upper bounded
by the $ \lceil\frac{d}{2} + \frac32\rceil $th derivative.
For the polynomials of projection, we provide a different way to evaluate
a tighter upper bound of its norm.

The tool we use here is spherical harmonics.
The kernel derived the uniform ReLU features $ (\mathrm{ReLU},\tau_d) $
can be explicitly written as follows \citep{Cho2009},
\[
\int_{\mathbb{S}^{d}}\mathrm{ReLU}(\omega\cdot\tilde{x})\mathrm{ReLU}(\omega\cdot\tilde{x}')\,d\tau_{d}(\omega)=\frac{1}{2(d+1)\pi}\left[\sqrt{1-(\tilde{x}\cdot\tilde{x}')^{2}}+\left(\pi-\arccos(\tilde{x}\cdot\tilde{x}')\right)\tilde{x}\cdot\tilde{x}'\right]\,,
\]

Note that we rewrite the kernel function in a different but equivalent
form compared to the formula in Cho's work to emphasize that it is
of the form $k(\tilde{x}\cdot\tilde{y})$,
so called dot product kernels. The kernel function $k(s)$ has the
Taylor expansion
\[
\frac{1}{2(d+1)}\left(\frac{1}{\pi}+\frac{1}{2}s+\frac{1}{2\pi}s^{2}+\frac{1}{\pi}\sum_{k=2}^{\infty}\frac{(2k-3)!!}{(2k)!!}\frac{1}{2k-1}s^{2k}\right)
=\sum_{j=0}^{\infty}a_{j}s^{j}\,.
\]
Then by \cite{Azevedo2014}, we know that $k(\tilde{x}\cdot\tilde{y})$
has Mercer type expansion
\[
k(\tilde{x}\cdot\tilde{y})=\sum_{i=1}^{\infty}\sum_{j=1}^{N(i,d)}\lambda_{i}Y_{i,j}^{d}(\tilde{x})Y_{i,j}^{d}(\tilde{y})\,,
\]
where $\{Y_{i,j}^{d}\}$ are the spherical harmonics on $\mathbb{S}^{d}$,
$N(i,d)$ is the dimension of the eigenspace corresponding to $i$th
eigenvalue and
\begin{align*}
\lambda_{i} & =|\mathbb{S}^{d-1}|\int_{-1}^{1}k(s)P_{i}^{d}(s)(1-s^{2})^{\frac{d-2}{2}}\,ds\\
 & =|\mathbb{S}^{d-1}|\int_{-1}^{1}\sum_{j=0}^{\infty}a_{j}s^{j}P_{i}^{d}(s)(1-s^{2})^{\frac{d-2}{2}}\,ds\\
 & =\sum_{j=0}^{\infty}a_{j}c(i,j,d)\,.
\end{align*}
$c(i,j,d)>0$ when $i\le j$ and $i\equiv j\mod2$ and $c(i,j,d)=0$
otherwise. Now consider the function $f(\tilde{x})=\alpha(\beta\cdot\tilde{x})^{p}$.
Without loss of generality, we assume that $\beta\in\mathbb{S}^{d}$.
Its projection onto the subspace spanned by $Y_{i,j}^{d}$ is given
by the Funk-Hecke formula
\[
\int_{\mathbb{S}^{d}}\alpha(\beta\cdot\tilde{x})^{p}Y_{i,j}^{d}(\tilde{x})\,d\sigma_{d}(\tilde{x})=\alpha c(i,p,d)Y_{i,j}^{d}(\beta)\,.
\]
In other words, for even number $p$ or $p=1$,
\[
\alpha(\beta\cdot\tilde{x})^{p}=\sum_{i=1}^{p}\sum_{j}\alpha c(i,p,d)Y_{i,j}^{d}(\beta)Y_{i,j}^{d}(\tilde{x})\,.
\]
Then the RKHS norm of $f$ is given by
\[
\Vert f\Vert_{\mathrm{RKHS}}^{2}=\alpha^{2}\sum_{i=1}^{p}\sum_{j}\frac{c^{2}(i,p,d)}{\lambda_{i}}\left(Y_{i,j}^{d}(\beta)\right)^{2}\,.
\]
Note that $\lambda_{i}\ge a_{p}c(i,p,d)$. So we have
\begin{align*}
\Vert f\Vert_{\mathrm{RKHS}}^{2} & \le\alpha^{2}\sum_{i=1}^{p}\sum_{j}a_{p}^{-1}c(i,p,d)\left(Y_{i,j}^{d}(\beta)\right)^{2}\\
 & =a_{p}^{-1}\alpha^{2}(\beta\cdot\beta)^{p}\\
 & =a_{p}^{-1}\alpha^{2}\,.
\end{align*}
Since $a_{p}\ge\frac{1}{2(d+1)\pi p^{2}}$, for $p=1$ or even $p$, we
have the following lemma.
\begin{lemma}
  \label{lem:proj_fn_norm}
$\Vert\alpha(\beta\cdot\tilde{x})^{p}\Vert_{\tilde{\mathcal{H}}}\le\sqrt{2(d+1)\pi}\alpha p$.
\end{lemma}

Using this result, we can further justify whether a function
belongs to the $ \tilde{\mathcal{H}} $ based on its Taylor expansion.

Next, we show that the function $ g(x)=\alpha(\beta\cdot(x,1))^p(1+\Vert x\Vert^2)^{(1-p)/2} $
defined on $ \mathbb{R}^d $ belongs to $ \mathcal{H}_{\mathrm{ReLU},\tau_d} $.
Assume that $ f(\tilde{x}) = \alpha(\beta\cdot\tilde{x})^p $, then
\begin{align*}
    \alpha(\beta\cdot(x,1))^p(1+x^2)^{(1-p)/2} & =
    \sqrt{\Vert x\Vert^2+1}f\left(\frac{(x,1)}{\sqrt{\Vert x\Vert^2+1}}\right) \\
    & = \int_{\mathbb{S}^d} h(\omega)\mathrm{ReLU}(\omega\cdot(x,1))~\mathrm{d}\tau_d(\omega)\,, \\
\end{align*}
for some $ h \in L^2(\mathbb{S}^d,\tau_d) $.
So
\[ \Vert g\Vert_{\mathrm{ReLU},\tau_d} \le \Vert h\Vert_{L^2(\tau_d)}
= \Vert f\Vert_{\tilde{\mathcal{H}}} \le \sqrt{2(d+1)\pi}\alpha p\,. \]
When $ p = 2, \alpha = 1 $ and $ \beta = \vec{e}_j $, we have
\[
    \left\Vert\frac{x_j^2}{\sqrt{1+\Vert x\Vert^2}}\right\Vert
    _{\mathrm{ReLU},\tau_d} \le 2\sqrt{2(d+1)\pi}\,.
\]
And
\begin{align}
    \left\Vert\sqrt{1+\Vert x\Vert^2}\right\Vert_{\mathrm{ReLU},\tau_d}
    & \le \sum_{j=1}^d \left\Vert\frac{x_j^2}
    {\sqrt{1+\Vert x\Vert^2}}\right\Vert_{\mathrm{ReLU},\tau_d} \notag\\
    & \le 2\sqrt{2\pi}(d+1)^{3/2}\,.\label{eq:radius_rkhs_norm}
\end{align}
Similarly, when $x\in \mathbb{S}^{d-1}\times\mathbb{S}^{d-1}$, we have
\begin{equation*}
  \left\Vert\alpha(\beta\cdot(x,1))^p\right\Vert_{\mathrm{ReLU},\tau_{2d}}
  \le |\beta|^p3^{p/2}\sqrt{(4d+2)\pi}\alpha p\,.
\end{equation*}
We further have
\begin{align}
  \left\Vert x_{1:d}\cdot x_{d+1:2d}\right\Vert_{\mathrm{ReLU},\tau_{2d}} & \le
  \sum_{j=1}^d \left\Vert x_j x_{j+d}\right\Vert_{\mathrm{ReLU},\tau_{2d}}
  \notag\\
  & = \sum_{j=1}^d \frac12\left\Vert\left((x\cdot(\vec{e}_j+\vec{e}_{j+d}))^2
  -(x\cdot \vec{e}_j)^2-(x\cdot\vec{e}_{j+d})^2\right)
  \right\Vert_{\mathrm{ReLU},\tau_{2d}} \\
  & \le 12d\sqrt{(4d+2)\pi}\,.\label{eq:dotprod_rkhs_norm}
\end{align}

To show the depth separation results, we need only prove the following claim.
\begin{proposition}
There exist $ f:\mathbb{S}^{d-1}\times\mathbb{S}^{d-1}\to\mathbb{R} $
in $ \mathcal{H}_{\mathrm{ReLU},\tau_{2d}} $,
$ g:\mathbb{R}\to\mathbb{R} $
in $ \mathcal{H}_{\mathrm{ReLU},\tau_1} $, both with norm less than
$ \mathrm{poly}(d) $, a probability distribution $\mathbb{P}$
on $\mathbb{S}^{d-1}\times\mathbb{S}^{d-1}$, and a constant $c$,
such that for any 2-layer ReLU networks $ h $
with number of nodes fewer than $ 2^{\Omega(d\log(d))} $ and weights less than $2^d$,
\[
    \Vert h-g\circ f \Vert_{L^2(\mathbb{P})} > c\,.
\]
\end{proposition}
This proposition implies that no functions
in $ \mathcal{H}_{\mathrm{ReLU},\tau_{2d}} $
with norm less than $ \mathrm{poly}(d) $ can approximate this
composite function. Otherwise it contradicts the fact that any
functions in the $ \mathcal{H}_{\mathrm{ReLU},\tau_{2d}} $
with norm less than $ R $ can be approximated with an error of
$ \epsilon $ by a 2-layer ReLU network with $ O(R^2/\epsilon^2) $ nodes
and weights less than a certain absolute constant by Corollary~\ref{cor:main}.
And for the case of odd dimensions $d$ in Proposition~\ref{prop:separation},
we can just set $\mathcal{X}$ to be $\mathbb{S}^{(d-1)/2}
\times\mathbb{S}^{(d-1)/2}$ and use the construction for even dimensions.
Now we prove the above proposition.
\begin{proof}
  We use the construction of Example~2 in \cite{Daniely2017}, where
  $f(x)=\sum_{i=1}^d x_i\cdot x_d+i$ and $g(t)=\sin(\pi d^3 t)$, and the
  data distribution is the uniform distribution over $\mathbb{S}^{d-1}\times
  \mathbb{S}^{d-1}$. $f$ belongs to $\mathcal{H}_{\mathrm{ReLU},\tau_{2d}}$ with
  RKHS norm less than $12d\sqrt{(4d+2)\pi}$ as shown by
  Eq~\ref{eq:dotprod_rkhs_norm}. $g$ belongs to $\mathcal{H}_{\mathrm{ReLU},\tau_1}$ with RKHS norm less than $Cd^3$ where
  $C$ is an absolute constant by Proposition~5.
\end{proof}

By using Eldan and Shamir's result, we can also construct a function from
$\mathbb{R}^d\to \mathbb{R}$ composed by two functions from
$\mathcal{H}_{\mathrm{ReLU},\tau_d}$ and $\mathcal{H}_{\mathrm{ReLU},\tau_1}$
respectively, which cannot be approximated by any 2-layer ReLU networks with
$\mathrm{poly}(d)$ nodes. The procedure is as follows.

% \begin{proof}
  For any $ \omega_i \in \mathbb{S}^{d} $ and $ c_i \in \mathbb{R} $,
  we define
  \[
    f_i(t) = c_i \mathrm{ReLU}(t(1-(\omega_i)_{d+1}^2)^{1/2}
    + \omega_{d+1})\,.
  \]
  Then
  \[
    f_i\left(x\cdot\frac{(\omega_i)_{1:d}}{\Vert(\omega_i)_{1:d}\Vert
    }\right) = c_i\mathrm{ReLU}(\omega_i\cdot(x,1))\,,
  \]
  where we use the convention that $ 0/\Vert 0\Vert = 0 $.
  Then by Proposition~13 of \cite{Eldan2016}, we know for any
  $\omega_i$ and $c_i$, there exists a function $ \tilde{g} $
  and data distribution $ \mathbb{P} $ such that
  \[
    \left\Vert\sum_{i=1}^N f_i\left(x\cdot\frac{(\omega_i)_{1:d}}{\Vert(\omega_i)_{1:d}
    \Vert}\right) - \tilde{g}(x)\right\Vert_{L^2(\mathbb{P})} \ge \frac{\delta}
    {\alpha}\,,
  \]
  where $ \delta $ and $ \alpha $ are two constants.
  Now we construct $ \tilde{g} $ as a composite function $ g\circ f $.
  We choose $ f(x) = \sqrt{\Vert x\Vert^2 + 1} $. By Equation~\ref{eq:radius_rkhs_norm},
  $ \Vert f\Vert_\mathcal{H} \le 2\sqrt{2\pi}(d+1)^{3/2} $.
  The data distribution is also chosen to be the squared Fourier transform of
  the indicator function of unit volume ball,
  the same with \cite{Eldan2016}. Then we construct $ g $ based on
  the function $ \tilde{g} $ in Proposition~13 in \cite{Eldan2016}.
  By Lemma~12 in Eldan and Shamir's work, we know that there exists
  $ N $-Lipschitz function  $ g_1 $ where $ N=c(\alpha d)^{3/2} $
  supported on $ [\alpha\sqrt{d}, 2\alpha\sqrt{d}] $ with range
  in $ [-1,1] $ such that
  \[
    \Vert g_1(\Vert x\Vert) - \tilde{g}(\Vert x\Vert)
    \Vert_{L^2(\mathbb{P})}^2 \le \frac{3}{\alpha^2\sqrt{d}}\,.
  \]
  Now we define $ g_2(t) = g_1(\sqrt{t^2-1}) $. It is supported on
  $ [\sqrt{\alpha^2 d+1}, \sqrt{4\alpha^2 d+1}] $ and $ 3N $-Lipschitz.
  And we have $ g_2(\sqrt{\Vert x\Vert^2+1}) = g_1(\Vert x\Vert) $.
  The last step is to find an approximator from
  $ \mathcal{H}_{\mathrm{ReLU},\tau_d} $ for $ g_2 $. This can be done
  be applying Proposition~6 in \cite{Bach2017a}, which implies
  in our setup that there exists a function $ g $ with
  $ \Vert g\Vert_{\mathcal{H}} \le M $ and
  \[
    \sup_{|t| \le \sqrt{4\alpha^2 d+1}}
    | g(t) - g_2(t) |
    \le C N\sqrt{4\alpha^2 d+1}
    \left(\frac{M}{3N\sqrt{4\alpha^2 d+1}}\right)^{-1/2}\,,
  \]
  where $ C $ is a universal constant.
  By setting $ M = (3N\sqrt{4\alpha^2 d+1})^{3/2} \alpha^2 \sqrt{d} $,
  we have
  \[
    \Vert g\circ f - \tilde{g} \Vert_{L^2(\mathbb{P})} \le O\left(
    \frac{1}{\alpha d^{1/4}}\right)\,.
  \]
  Therefore, we proved that
  \[
    \left\Vert\sum_{i=1}^N f_i\left(x\cdot\frac{(\omega_i)_{1:d}}{\Vert(\omega_i)_{1:d}
    \Vert}\right) - g\circ f(x)\right\Vert_{L^2(\mathbb{P})} \ge \frac{\delta}
    {2\alpha}\coloneqq c\,,
  \]
% \end{proof}

Eldan and Shamir's hard-to-approximate result does not rely on
any magnitude constraint over weights of networks,
in which sense it is stronger than Daniely's result.
The function constructed by Eldan and Shamir is compactly supported, but the
data distribution is over the whole $\mathbb{R}^d$, and thus it does not
rule out the possibility that such a target composite function can be
well-approximated by a 2-layer ReLU network against
any data distribution whose support is the same with the target function.

%% file: pf_barron.tex
%!TEX root = main.tex
\section{Approximation Results of Single-Hidden-Layer ReLU Networks}
\label{app:main}

Before the proof of the the multi-layer approximation result, We first prove 
a single-hidden-layer approximation result for ReLU networks with bounds on the inner weights. 
We consider the functions with the following representation
\[f(x)\coloneqq \int_{\mathbb{R}^{d+1}}\mathrm{ReLU}(\omega\cdot(x,1))~\mathrm{d}\rho(\omega)\,,\]
where $\rho\in M(\mathbb{R}^{d+1})$. In particular, we denote the set of all such functions with 
$\int|\omega|\mathrm{d}|\rho|(\omega)\le R$ by $ \Lambda_R $. We further define that
$ \Lambda(\mathcal{X}) = \bigcup_{R>0} \Lambda_R(\mathcal{X}) $.

% Compared to Eq.~\eqref{eq:barron}, our definition of $\Lambda_R$ uses
% the Euclidean norm of $ \omega $ instead of
% $ \vert\cdot\vert_{\mathcal{X}} $. For compact $ \mathcal{X} $ with radius
% $ r $, $ \vert\omega\vert_{\mathcal{X}}\le r\vert\omega\vert $. On the other
% hand, as long as the convex hull of $ \mathcal{X}\cup\{0\} $ has non-empty
% interior, there always exists a constant $ c $ such that
% $ \vert\omega\vert_{\mathcal{X}}\ge c\vert\omega\vert $.
% So in most cases $ |\cdot| $ and $ |\cdot|_\mathcal{X} $ only differ by
% a constant factor.
% We should note that the integral representation of $ f $ in Eq.~\eqref{eq:relu-feature}
% is not unique. There exist cases where the same $ f $ can expressed
% by several measures, but only some of them satisfy the constraint of $ \Lambda_R(\mathcal{X}) $.
% The original Barron's class exhibits a similar situation. For example,
% in the case where the Fourier transform of an extension of $ f $
% to the entire space is integrable, it provides a complex measure that generates
% $ f $ using Eq.~\eqref{eq:barron}, and such an extension is not unique.

Some properties of $ \Lambda_R(\mathcal{X}) $ and $ \Lambda(\mathcal{X}) $
are given below.

%\ambuj{Status of the generality level of $\sigma$: in the result below it seems general but the rest of the results in this section are for
%ReLU and leaky ReLU. We should be very clear about the level of generality of $\sigma$ in our results.}

\begin{proposition}
  \label{prop:lambda_R}
  \leavevmode
  \begin{enumerate}
    \item $ \Lambda(\mathcal{X}) $ consists of continuous functions.
    \item Functions in $ \Lambda_R(\mathcal{X}) $ are $ R $-Lipschitz.
    \item Assume that $ \mu $ is a probability measure on $ \mathbb{R}^{d+1} $ 
    with $ \int_\Omega |\omega|^2~\mathrm{d}\mu(\omega) = M_2 $. Then
    $ \{ f: \Vert f\Vert_{\mathrm{ReLU},\mu} \le R\}
      \subset \Lambda_{R\sqrt{M_2}}(\mathcal{X}) $.
    \item $ \Lambda(\mathcal{X}) $ is universal.
  \end{enumerate}
\end{proposition}

\begin{proof}
  \leavevmode
  \begin{enumerate}
    \item This is implied by the second statement.
    \item $ \mathrm{ReLU} $ is $ 1 $-Lipschitz.
    For $ f $ in $ \Lambda_C(\mathcal{X}) $ and $ x_1,x_2 $ in $ \mathcal{X} $,
    \begin{align*}
      \vert f(x_1) - f(x_2) \vert & = \left\vert \int \left(\mathrm{ReLU}(\omega\cdot (x_1,1))
      - \mathrm{ReLU}(\omega\cdot (x_2,1)\right)~\mathrm{d}\rho(\omega) \right\vert \\
      & \le \int \left\vert \mathrm{ReLU}(\omega\cdot (x_1,1)
      - \mathrm{ReLU}(\omega\cdot (x_2,1) \right\vert~\mathrm{d}\vert\rho\vert(\omega) \\
      & \le \int \vert\omega\cdot (x_1 - x_2,0)\vert~\mathrm{d}\vert\rho\vert(\omega) \\
      & \le R\vert x_1-x_2\vert\,.
    \end{align*}
    \item For any $ f $ in the ball of radius $ R $ of the RKHS $ \mathcal{H}_{\mathrm{ReLU},\mu} $,
    we have that
    \begin{equation*}
      f(x) = \int \sigma(\omega\cdot (x,1))g(\omega)
      ~\mathrm{d}\mu(\omega)\quad\text{and}
      \quad \int g^2(\omega)~\mathrm{d}\mu(\omega) \le R^2\,.
    \end{equation*}
    Then
    \begin{align*}
      \int \vert\omega\vert \vert g(\omega)\vert~\mathrm{d}\mu(\omega)
      & \le \left(\int \omega\cdot\omega~\mathrm{d}\mu(\omega)
      \int g^2(\omega)~\mathrm{d}\mu(\omega)\right)^{1/2} \\
      & \le \sqrt{M_2}R\,.
    \end{align*}
    \item 
    This can be shown by the fact that all ReLU networks belong to
    $ \Lambda $ and the set of ReLU networks is universal.
  \end{enumerate}
\end{proof}
Usually $ \mathcal{H}_{\mathrm{ReLU},\mu} $ is strictly included in $ \Lambda $ and
so is $ \Lambda $ in $ C(\mathcal{X}) $.
Indeed, when $ \mu $ is absolutely continuous with respect to Lebesgue measure,
for any fixed $ \omega $, $ \mathrm{ReLU}(\omega\cdot (x,1)) $
belongs to $ \Lambda(\mathcal{X}) $, but not $ \mathcal{H}_{\mathrm{ReLU},\mu} $.
And $ \Lambda(\mathcal{X}) $ only contains Lipschitz functions,
but there clearly exist continuous functions over $ \mathcal{X} $ that are not Lipschitz.
Actually $ \Lambda_R $ is also strictly smaller than the set of 
$ R $-Lipschitz functions by the depth separation result in \cite{Eldan2016}.

For the functions in $ \Lambda_R(\mathcal{X}) $, we can approximate them by
ReLU networks. Stronger than Barron's approximation theorem,
our theorem provides $ O(1) $ upper bounds for both inner and outer weights of the
single-hidden-layer neural networks.
\begin{theorem}
  \label{thm:main}
  Assume that $ \mathcal{X} $ is bounded by a ball of radius $ r $.
  $ \mathbb{P} $ is a probability measure on $ \mathcal{X} $.
  For any $ f\in\Lambda_R(\mathcal{X}) $, there exists $ g(x)=\sum_{i=1}^N
  c_i\mathrm{ReLU}(\tilde{\omega}_i\cdot (x,1)) $
  with $ \vert c_i\vert = \tilde{R}/N \le R/N $, and $ \vert\tilde{\omega}_i\vert = 1 $
  for all $ i $, such that $ \Vert f-g\Vert_{L^2(\mathbb{P})} \le R\sqrt{r^2+1}/\sqrt{N} $.
\end{theorem}
This theorem shows that for ReLU networks to approximate
the functions in $ \Lambda_R(\mathcal{X}) $ within an error of $ \epsilon $,
only $ O(C/\epsilon^2) $ nodes are required.
Moreover, every outer weights in front of nodes are bounded by
$ O(\epsilon^2) $ and only differ by signs, and the inner weights are of unit length.
Compared with the theorem in \cite{Barron1993}, where outer weights are bounded by a constant under
$ \ell^1 $ norm, but the bound on inner weights depends inversely on the approximation
error, our extra bounds on inner weights largely shrink the search area for approximators.
This improvement comes from the fact that functions in $ \Lambda_R(\mathcal{X}) $ are defined
by the transformation induced by ReLU functions, which is homogeneous of degree 1. The proof is an
application of Maurey's sparsification lemma \citep{Pisier1980}.

\begin{lemma}[Maurey's sparsification]
\label{lem:maurey}
Assume that $ X_1,\ldots,X_n $ are $ n $ i.i.d. random variables with values in the unit ball
of a Hilbert space. Then with probability greater than $ 1-\delta $, we have
\begin{equation}
  \left\Vert \frac{1}{n}\sum_{i=1}^n X_i - \mathbb{E}X_1 \right\Vert
  \le \frac{1}{\sqrt{n}}\left(1+\sqrt{2\log\frac{1}{\delta}}\right)\,.
\end{equation}
Furthermore, there exists $ x_i,\ldots,x_n $ in the unit ball such that
\begin{equation}
  \left\Vert \frac{1}{n}\sum_{i=1}^n x_i - \mathbb{E}X_1 \right\Vert
  \le \frac{1}{\sqrt{n}}\,.
\end{equation}
\end{lemma}

With this lemma, we can prove Theorem~\ref{thm:main}.
\begin{proof}
  If $ f $ is constantly 0, we only need to choose $ c_i $
  to be $ 0 $ for all $ i\in[N] $. Now assume that $ f $ is not constantly 0.
  Then we can always restrict the feature space to
  $ S = \mathbb{R}^{d+1}\backslash\{0\} $ and have
  $ 0 < \int_S \vert\omega\vert~\mathrm{d}\vert\rho\vert(\omega) = \tilde{R} \le R $.
  $ f $ can be written into the following form
  \begin{equation*}
    f(x) = \int_S \tilde{R}\tau(\omega)\mathrm{ReLU}\left(\frac{\omega\cdot
    (x,1)}{\vert\omega\vert}\right)
    \frac{\vert\omega\vert}{\tilde{R}}~\mathrm{d}\vert\rho\vert(\omega)\,,
  \end{equation*}
  where $ \tau(\omega) $ equals 1 when $ \omega $ belongs to the positive set of $ \rho $
  and -1 when it belongs to the negative set of $ \rho $. Since
  $ \frac{\vert\omega\vert}{\tilde{R}}~\mathrm{d}\vert\rho\vert(\omega) $ is a probability
  measure and
  \begin{equation*}
    \left\Vert \tilde{R}\tau(\omega)\mathrm{ReLU}\left(\frac{\omega\cdot
    (x,1)}{\vert\omega\vert}\right)\right\Vert_{L^2(\mathbb{P})}
    \le \tilde{R}a_\mathrm{ReLU}\sqrt{r^2 + 1}\,,
  \end{equation*}
  we can apply Lemma~\ref{lem:maurey} and get the conclusion that there exists
  $ \omega_i $s such that
  \begin{equation*}
    \left\Vert f-\frac{1}{N}\sum_{i=1}^N
    \tilde{R}\tau(\omega_i)\mathrm{ReLU}\left(\frac{\omega_i\cdot
    (x,1)}
    {\vert\omega_i\vert}\right)\right\Vert_{L^2(\mathbb{P})}
    \le \frac{\tilde{R}a_\mathrm{ReLU}\sqrt{r^2+1}}{\sqrt{N}}\,.
  \end{equation*}
  By setting $ \tilde{\omega}_i = \omega_i/\vert\omega_i\vert $ and
  $ c_i = \tilde{R}\tau(\omega_i)/N $, the statement is proved.
\end{proof}

A direct corollary of this theorem is as follows.
\begin{corollary}
\label{cor:main}
For any $f\in\mathcal{H}_{\mathrm{ReLU},\tau_d}$
with $\Vert f\Vert_{\mathrm{ReLU},\tau_d}\le R$,
there exists $ g(x)=\sum_{i=1}^N c_i\mathrm{ReLU}(\tilde{\omega}_i\cdot (x,1)) $
with $ \vert c_i\vert = \tilde{R}/N \le R/N $, 
and $ \vert\tilde{\omega}_i\vert = 1 $
for all $ i $, such that $ \Vert f-g\Vert_{L^2(\mathbb{P})} \le R\sqrt{r^2+1}/\sqrt{N} $.
\end{corollary}
% Even though Barron's class $ \Gamma_R(\mathcal{X}) $ also consists of Lipschitz functions,
% Barron's proof requires piecewise constant functions in an intermediate step
% when constructing the sigmoidal networks to approximate target functions,
% and thus loses the Lipschitz property. Even if we instead consider the more direct
% approximators such as linear combination of sinusoidal nodes, it is still
% unclear how the inner weights can be controlled.

% The networks of cosine and sine nodes have the same form of
% random Fourier features models. The relation between random Fourier features models
% and Barron's class is exactly parallel to the relation between random
% ReLU features models and ReLU feature class.
%

% The following theorem, as an analogy to Theorem~3.5 in \cite{Lee2017},
% shows that for the composition of functions in $ \Lambda_R $,
% we can use a multi-layer ReLU network to approximate it. And all the weights in
% the neural network can be controlled by some constant related to $ R $ and dimensions
% of functions.
% This result justifies what type of functions can be learned when
% the weight matrices of multi-layer ReLU networks are constrained by some constants.
% The proof can be found in Appendix~\ref{app:multi}.

% It is not surprising that as with Barron's work, the key step of the proof is also
% the following Maurey's sparsification lemma.

%% file: pf_multi.tex
%!TEX root = main.tex
\section{Proof of Approximation Results of Multi-Layer ReLU Networks}
\label{app:multi}

We extend the definition of $\Lambda_R$ in Appendix~\ref{app:main} 
to vector-valued functions.
\begin{definition}
  For a map $f$ from $ \mathbb{R}^m $ to $ \mathbb{R}^n $, we say that it belongs
  to the class $ \Lambda_R(\mathcal{X}) $ if each component $ (f)_i \in \Lambda_{R_i}(\mathcal{X}) $
  for $ 1\le i\le n $ and $ \sum_i R_i^2 \le R^2 $.
\end{definition}
Note that Proposition~\ref{prop:lambda_R}~(2) still holds for the vector valued
function class $ \Lambda_R $,
and any vector valued function $f\in\mathcal{H}_{\mathrm{ReLU},\tau_d}$ 
with $\Vert f\Vert_{\mathrm{ReLU},\tau_d}\le R$ also belongs to 
$\Lambda_R$.
Therefore, to prove the approximation result in Proposition~\ref{prop:multi}, we need only prove
it for $\Lambda_R$.
\begin{theorem}
\label{thm:multi}
  Assume that for all $ 1\le i\le L+1 $, $ K_i $ is a compact set with
  radius $ r $ in $ \mathbb{R}^{m_i} $ and $ B^{m_i} $ is the unit ball
  in $ \mathbb{R}^{m_i} $ . $ K_1 = \mathcal{X} $.
  $ f_i:K_i+sB^{m_i}\to K_{i+1} $ belongs to
  $ \Lambda_{R_i}(K_{i}+sB^{m_{i}}) $ for $ 1\le i\le L $. Then for any
  $ 1\le\ell\le L $ and $ \epsilon>0 $, there exists a set $ S_\ell\subset K_1 $ with
  \begin{equation*}
    \mathbb{P}(S_\ell) \ge 1 - \frac{\epsilon^2}{s^2}\sum_{i=1}^{\ell-1}
    \frac{i^2}{\prod_{j=i+1}^LR_j^2}\,,
  \end{equation*}
  and an $ \ell $-layer neural networks $ g_{1:\ell} $ where
  \begin{align*}
    g_i & :\mathbb{R}^{m_{i}}\to\mathbb{R}^{m_{i+1}} \\
    (g_i(x))_j & = \sum_{k=1}^{N_i}c_{i,j,k}
    \sigma(\omega_{i,j,k}\cdot (x,1)) \\
    \vert c_{i,j,k}\vert & = \frac{\tilde{R}_{i,j}}{N_i}\le \frac{R_{i,j}}{N_i} \\
    \vert \omega_{i,j,k}\vert & \le 1 \\
    N_i & = \frac{\prod_{j=i}^LR_j^2((r+s)^2+1)}{\epsilon^2}\,,
  \end{align*}
  such that
  \begin{equation*}
    \left(\int_{S_\ell}\vert f_{1:\ell}-g_{1:\ell}\vert^2~\mathrm{d}\mathbb{P}\right)^{1/2}
    \le \frac{\ell\epsilon}{\prod_{i=\ell+1}^LR_i}\,,
  \end{equation*}
  with the convention that $\prod_{i=L+1}^L R_i=1$.
  Moreover, the $ i $th layer of the neural network $ g_{1:\ell} $
  contains $ m_{i+1}N_i $ nodes. The weight matrix from layer $ i $
  to layer $ i+1 $, denoted by $ W_{i\to i+1} $, has Frobenius norm
  bounded by $ \sqrt{m_{i+2}} $ for $1\le i\le L-1$, 
  $\Vert W_{0\to1} \Vert_F \le \prod_{i=1}^L R_i
  \sqrt{m_2((r+s)^2+1)}$, and
  $\Vert W_{L\to L+1} \Vert_F \le \sqrt{(r+s)^2+1}$.
  Each bias term is bounded by $ 1 $.
\end{theorem}
\begin{proof}
  For $ \ell=1 $, we construct the approximation $ g_1 $ for $ f_1 $ by applying
  Theorem~\ref{thm:main} to each component of $ f_1 $. First, set $ S_1=K_1 $.
  \begin{align*}
    \int_{S_1}\vert g_{1}(x)-f_{1}(x)\vert^2
    ~\mathrm{d}\mathbb{P}(x)
    & = \sum_{i=1}^{m_{2}} \int_{S_1}(f_{1}(x)-g_{1}(x))_i^2
    ~\mathrm{d}\mathbb{P}(x) \\
    & \le \sum_{i=1}^{m_{2}} \frac{R_{1,i}^2 ((r+s)^2+1)}{N_{1}} \\
    & \le \frac{R_{1}^2((r+s)^2+1)}{N_{1}}\,.
  \end{align*}
  Set
  \begin{equation*}
    N_1 = \frac{\prod_{i=1}^LR_i^2((r+s)^2+1)}{\epsilon^2}
  \end{equation*}
  and the conclusion holds for $ \ell=1 $.
  Assume that there exist $ g_{1:\ell} $ and $ S_{\ell} $ as described in the
  theorem. Define
  \begin{equation*}
    S_{\ell+1} = S_{\ell}\cap\left\{x\in K_1:\vert g_{1:\ell}(x)-f_{1:\ell}(x)
    \vert\le s\right\}\,.
  \end{equation*}
  Then by Markov's inequality and induction assumption,
  \begin{equation*}
    \mathbb{P}(S_{\ell+1})\ge 1-\frac{\epsilon^2}{s^2}\sum_{i=1}^{\ell-1}
    \frac{i^2}{\prod_{j=i+1}^L R_j^2 }
    - \frac{\epsilon^2\ell^2}{s^2\prod_{i=\ell+1}^LR_i^2}\,.
  \end{equation*}
  Then we want to construct $ g_{\ell+1} $ on $ g_{1:\ell}(S_{\ell+1}) $ to approximate
  $ f_{\ell+1} $, again by applying Theorem~\ref{thm:main} to each component of $ f_{\ell+1} $.
  Note that the measure we consider here is the push-forward of $ \mathbb{P} $ by
  $ g_{1:\ell} $, which is a positive measure with total measure less than $ 1 $.
  \begin{equation*}
    \int_{g_{1:\ell}(S_{\ell+1})}\vert g_{\ell+1}(x)-f_{\ell+1}(x)\vert^2
    ~\mathrm{d}g_{1:\ell}(\mathbb{P})(x) \le \frac{R_{\ell+1}^2((r+s)^2+1)}{N_{\ell+1}}\,.
  \end{equation*}
  And by triangle inequality,
  \begin{align*}
    & \hspace{-2em}\left(\int_{S_{\ell+1}}\vert g_{1:\ell+1}(x) - f_{1:\ell+1}(x)\vert^2
    ~\mathrm{d}\mathbb{P}(x)\right)^{1/2} \\
    & \le \left(\int_{S_{\ell+1}}\vert g_{\ell+1}\circ g_{1:\ell}(x) - f_{\ell+1}
    \circ g_{1:\ell}(x)\vert^2 ~\mathrm{d}\mathbb{P}(x)\right)^{1/2} \\
    & + \left(\int_{S_{\ell+1}}\vert f_{\ell+1}\circ g_{1:\ell}(x) - f_{\ell+1}
    \circ f_{1:\ell}(x)\vert^2 ~\mathrm{d}\mathbb{P}(x)\right)^{1/2} \\
    & \le \frac{R_{\ell+1}\sqrt{(r+s)^2+1}}{\sqrt{N_{\ell+1}}}
    + R_{\ell+1}\frac{\ell\epsilon}{\prod_{i=\ell+1}^{L}R_i}\,.
  \end{align*}
  Set
  \begin{equation*}
    N_{\ell+1} = \frac{((r+s)^2+1)\prod_{i=\ell+1}^LR_i^2}{\epsilon^2}\,,
  \end{equation*}
  and we get the upper bound
  \begin{equation*}
    \left(\int_{S_{\ell+1}}\vert g_{1:\ell+1}(x) - f_{1:\ell+1}(x)\vert^2
    ~\mathrm{d}\mathbb{P}(x)\right)^{1/2}
    \le \frac{(\ell+1)\epsilon}{\prod_{i=\ell+  2}^LR_i}\,.
  \end{equation*}

  Now let's examine the weight matrix and the bias term of the neural network
  $ g_{1:\ell} $. Note that the weight matrix from $ i $th layer to the $ (i+1) $th
  layer of $ g_{1:\ell} $ consists of $ c_{i,j,k} $ from $ g_{i} $
  and $ \omega_{i+1,j,k} $ from $ g_{i+1} $. For $ 1\le i\le L-1 $, the exact form
  of the weight matrices is given by the matrix product $ W_{i\to i+1} = AB $, where
  \begin{equation*}
    A_{i\to i+1} =
    \begin{pmatrix}
      (\omega_{i+1,1,1})_1 & \cdots & (\omega_{i+1,1,1})_{m_{i+1}} \\
      \vdots & \ddots & \vdots \\
      (\omega_{i+1,1,N_{i+1}})_1 & \cdots & (\omega_{i+1,1,N_{i+1}})_{m_{i+1}} \\

      (\omega_{i+1,2,1})_1 & \cdots & (\omega_{i+1,2,1})_{m_{i+1}} \\
      \vdots & \ddots & \vdots \\
      (\omega_{i+1,2,N_{i+1}})_1 & \cdots & (\omega_{i+1,2,N_{i+1}})_{m_{i+1}} \\

      \vdots & \ddots & \vdots \\

      (\omega_{i+1,m_{i+2},1})_1 & \cdots & (\omega_{i+1,m_{i+2},1})_{m_{i+1}} \\
      \vdots & \ddots & \vdots \\
      (\omega_{i+1,m_{i+2},N_{i+1}})_1 & \cdots & (\omega_{i+1,m_{i+2},N_{i+1}})_{m_{i+1}} \\
    \end{pmatrix}
  \end{equation*}
  \begin{equation*}
    B_{i\to i+1}
    \begin{pmatrix}
      \frac{\tilde{R}_{i,1}\tau_{i,1,1}}{N_i} \cdots
      \frac{\tilde{R}_{i,1}\tau_{i,1,N_i}}{N_i} & 0 \cdots 0 & 0 \cdots 0 \\
      0 \cdots 0 & \frac{\tilde{R}_{i,2}\tau_{i,2,1}}{N_i} \cdots
      \frac{\tilde{R}_{i,2}\tau_{i,2,N_i}}{N_i} & 0 \cdots 0 \\
      \vdots & \ddots & \vdots \\
      0 \cdots 0 & 0 \cdots 0 & \frac{\tilde{R}_{i,m_{i+1}}\tau_{i,m_{i+1},1}}{N_i} \cdots
      \frac{\tilde{R}_{i,m_{i+1}}\tau_{i,m_{i+1},N_i}}{N_i}
    \end{pmatrix}\,.
  \end{equation*}
  $ \tau_{i,j,k} $s are all $ \pm 1 $. Since $ \vert\omega_{i,j,k}\vert \le 1 $,
  each components are also bounded by $ 1 $. For $k$th row of $ W_{i\to i+1} $,
  \begin{align*}
    \vert (W_{i\to i+1})_k \vert^2 & =
    \sum_{j=1}^{m_{i+1}} N_i\frac{\tilde{R}_{i,j}^2}{N_i^2} \\
    & \le \frac{R_i^2}{N_i}\,.
  \end{align*}
  Hence $ \vert W_{i\to i+1} \vert_F^2 \le R_i^2 N_{i+1}m_{i+2}/N_i $.
  Plugging into the expressions of $ N_i $ and $ N_{i+1} $ and taking square root,
  we get the upper bound $ m_{i+2} $ as in the statement.
  For the bias term, we simply use the fact that it is a coordinate
  of $ \omega $ and thus bounded by $ 1 $.

  For the bottom layer the weight matrix is just given by $ A_{0\to1} $, its
  Frobenius norm can only be bounded by $ (N_1 m_2)^{1/2} $. On the other hand,
  the weight matrix of the top layer is given by $ B_{L\to L+1} $, whose
  Frobenius norm is bounded by $ R_L/\sqrt{N_L} $. Since $ N_i $ is of the
  scale $ 1/\epsilon^2 $, $ \vert A_{0\to1} \vert_F $ is of the scale $ \epsilon $
  while $ \vert B_{L\to L+1} $ is of the scale $ 1/\epsilon $. To resolve this
  issue, we can rescale $ A_{0\to1} $ down by a factor of $ \epsilon $. Because the
  ReLU nodes are homogeneous of degree $ 1 $, for the function $ g_{1:L} $
  remains unchanged, we need only scale all the bias terms in the
  intermediate layers down by the same factor $ \epsilon $, and scale up $ B_{L\to L+1} $
  by the factor of $ 1/\epsilon $. For $ \epsilon < 1 $, this rescaling keeps
  the weights matrices in the intermediate layers unchanged, the bias terms are
  still all less than $ 1 $. And
  \begin{equation*}
    \begin{split}
    \vert \epsilon A_{0\to1} \vert_F \le \prod_{i=1}^L R_i
    \sqrt{m_2((r+s)^2+1)}\,,\\
    \vert B_{L\to L+1}/\epsilon \vert_F \le \sqrt{(r+s)^2+1}\,.
    \end{split}
  \end{equation*}
\end{proof}

\begin{proposition*}\ref{prop:multi}.
Under the same assumptions of Theorem~\ref{thm:multi}, we have
\[\left(\int_{\mathcal{X}}|f_{1:L}-g_{1:L}|^2~\mathrm{d}\mathbb{P}\right)^{1/2}
\le C\epsilon\,,\]
where $C$ is a constant depending on $L,s,r,\{R_i\}_{i=1}^L$.
\end{proposition*}
\begin{proof}
  By Theorem~\ref{thm:multi}, we know that
  \begin{equation}
  \label{eq:sl}
    \left(\int_{S_L}\vert f_{1:L}-g_{1:L}\vert^2~\mathrm{d}\mathbb{P}\right)^{1/2}
    \le L\epsilon\,.
  \end{equation}
  and 
  \begin{equation*}
    \mathbb{P}(S_L) \ge 1 - \frac{\epsilon^2}{s^2}\sum_{i=1}^{L-1}
    \frac{i^2}{\prod_{j=i+1}^LR_j^2}\,.
  \end{equation*}
  By the bounds on the weights of $g_{1:L}$, know that 
  \[|g_{1:L}|\le \sqrt{r^2+1}\prod_{i=1}^L R_i\,.\]
  So
  \begin{align}
    \nonumber
    \left(\int_{\mathcal{X}\setminus S_L}
    \vert f_{1:L}-g_{1:L}\vert^2~\mathrm{d}\mathbb{P}\right)^{1/2}
    & \le \left(\int_{\mathcal{X}\backslash S_L}
    (\vert f_{1:L} \vert^2 +\vert g_{1:L}\vert)^2~\mathrm{d}\mathbb{P}\right)^{1/2}\\
    \nonumber
    & \le \left(r+(r^2+1)^{1/2}\prod_{i=1}^L R_i\right) \frac{\epsilon}{s}
    \left(\sum_{i=1}^{L-1}\frac{i^2}{\prod_{j=i+1}^LR_j^2}\right)^{1/2} \\
    & \le C\epsilon\,.\label{eq:x-sl}
  \end{align}
  By Eq.~\ref{eq:sl} and \ref{eq:x-sl}, we have
\[\left(\int_{\mathcal{X}}|f_{1:L}-g_{1:L}|^2~\mathrm{d}\mathbb{P}\right)^{1/2}
\le C\epsilon\,,\]
where $C$ is a constant depending on $L,s,r,\{R_i\}_{i=1}^L$.
\end{proof}

%% file: extra_img.tex
\section{Extra Plots of Experiments}
\label{app:graphs}

In this section, we provide extra figures obtained from our experiments.

Figure~\ref{fig:data_illustration} shows the distribution
of synthetic data we use in the performance comparison between
random Fourier features and random ReLU.
\begin{figure}[h]
        \includegraphics[width=0.48\textwidth]{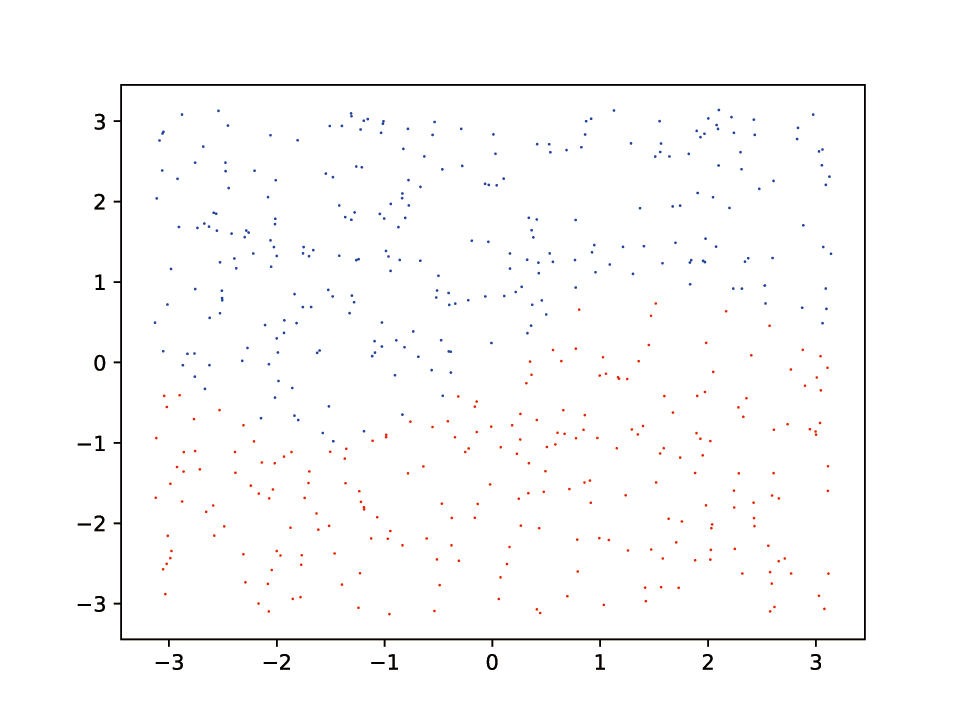}%
        \hfill%
        \includegraphics[width=0.48\textwidth]{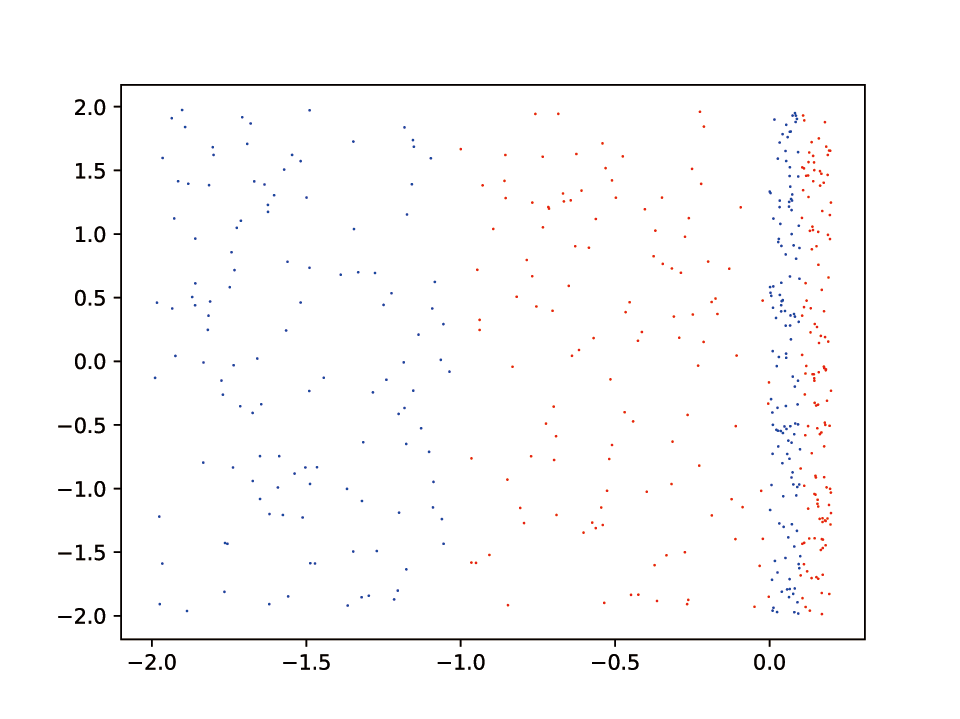}

        \includegraphics[width=0.48\textwidth]{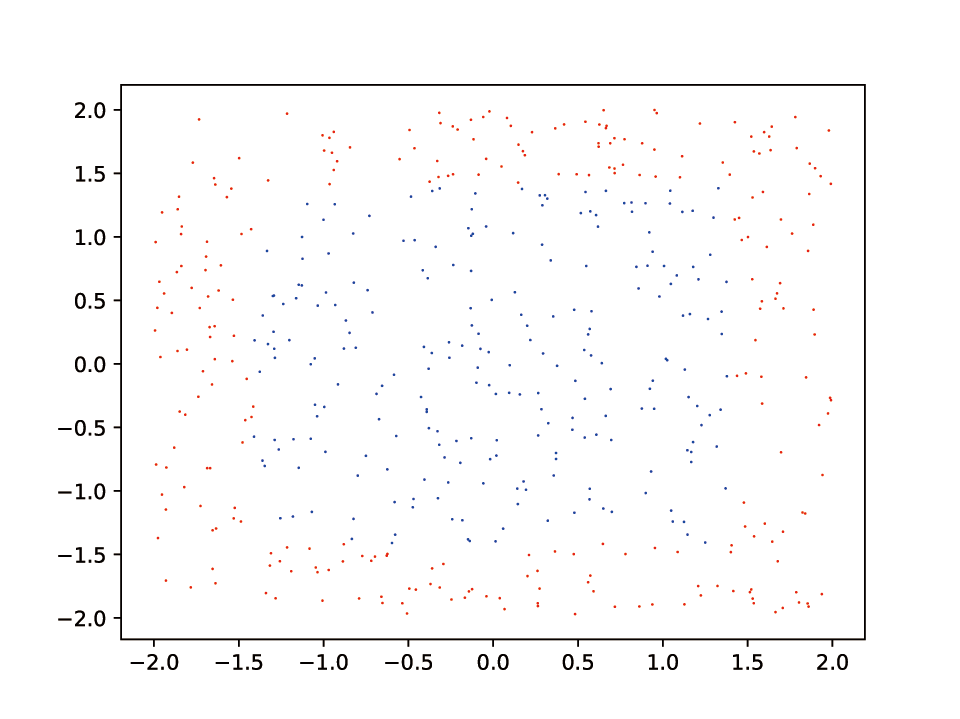}%
        \hfill%
        \includegraphics[width=0.48\textwidth]{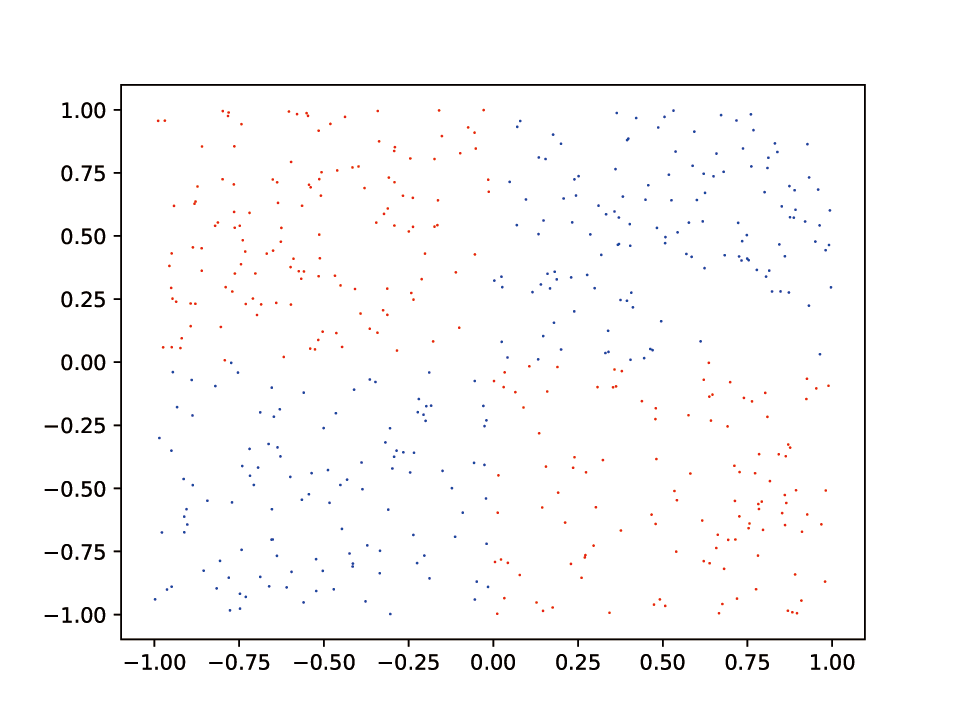}
    \caption{Illustrations on the distribution of synthetic datasets. Top left: sine. Top right: strips.
        Bottom left: square. Bottom right: checkboard.}
  \label{fig:data_illustration}
\end{figure}

Figure~\ref{fig:daniely-data} shows the distribution of the
data and the target function in our experiment on depth separation. It is
based on Daniely's construction.
\begin{figure}[htb]
        \includegraphics[width=0.48\textwidth]{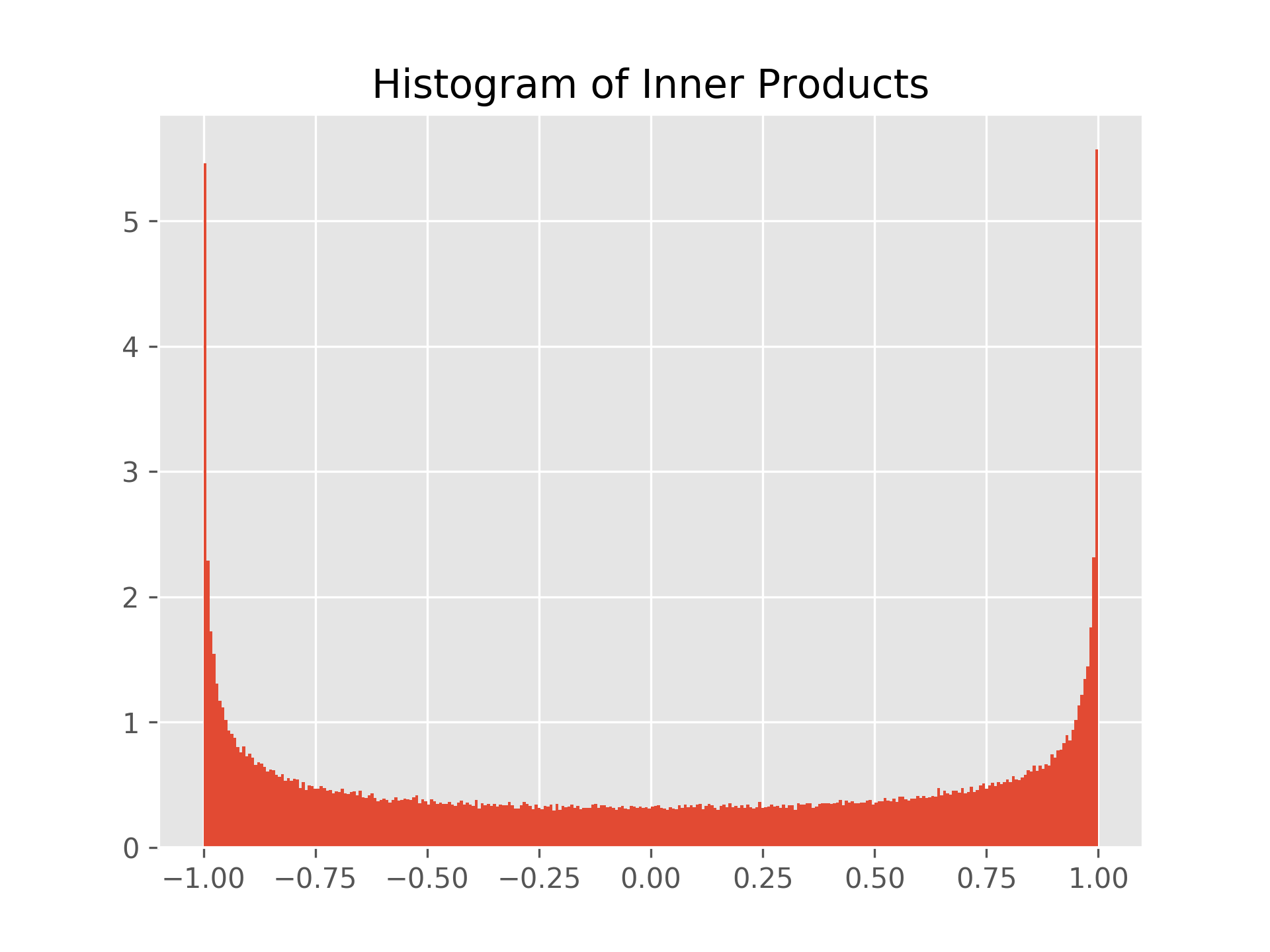}
        \hfill%
        \includegraphics[width=0.48\textwidth]{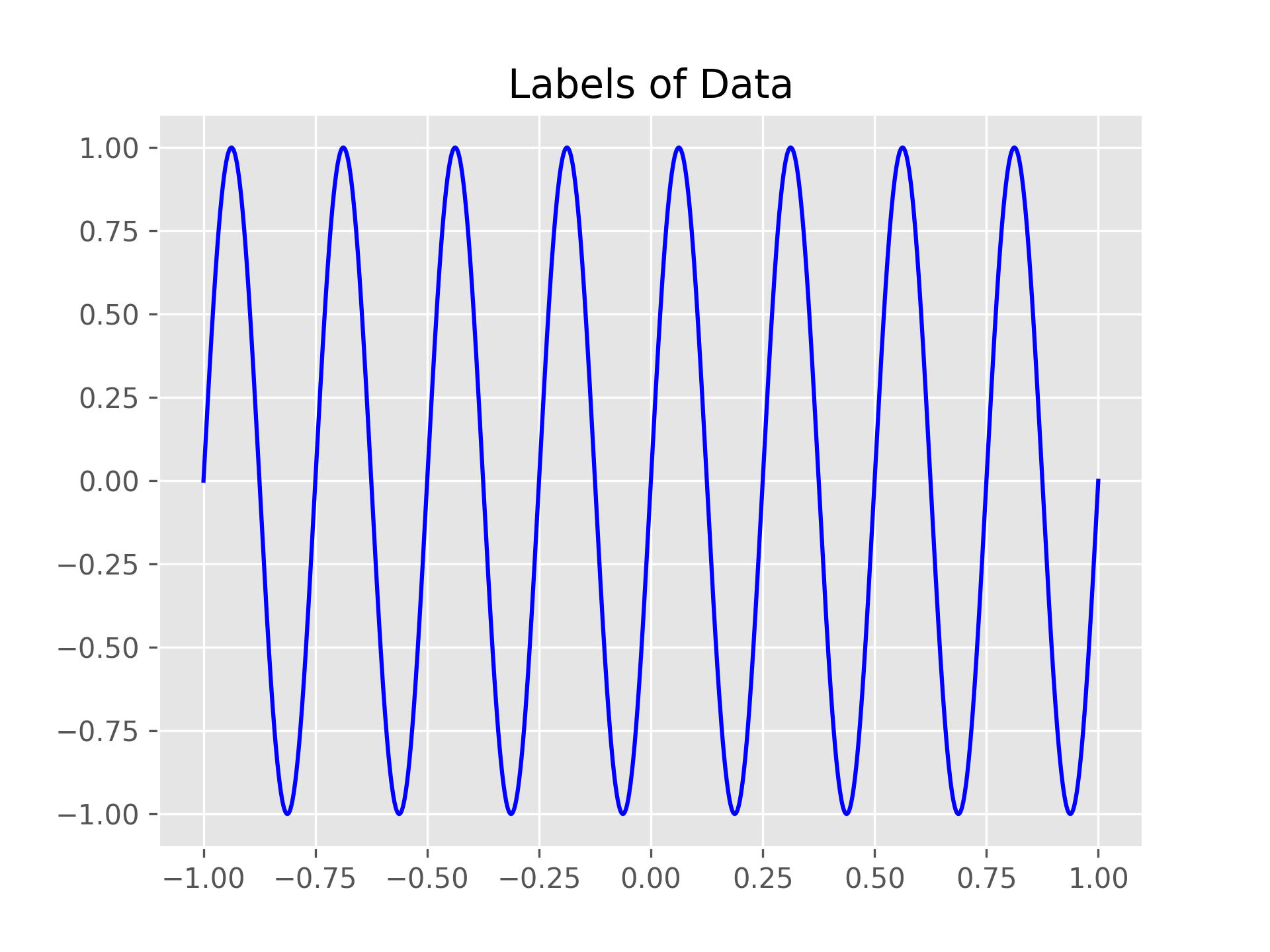}%
    \caption{Illustration of distributions of synthetic data of Daniely's example. Left: density of $x_1x_3+x_2x_4$. Right: Smoothed (blue) and unsmoothed (red) target labels against $x_1x_3+x_2x_4$.}
  \label{fig:daniely-data}
\end{figure}

We also test the depth separation phenomenon using Eldan and Shamir's
construction.
Figure~\ref{fig:eldan-data} shows the distribution of the
data and the target functions constructed by Eldan and Shamir.
The data distribution is rotation invariant over $ \mathbb{R}^4 $ with rapid oscillating radial density.
The classification target function $ g $ is $\pm 1$ at the high density region and $ 0 $ otherwise.
The regression target function $ \tilde{g} $ is obtained by mollifying $ g $ with the smooth bump function.
%and the functions learned by them are also quite similar (see Appendix~\ref{app:graphs}).
% The performance of the three models in classification tasks is consistent with that in Figure~\ref{fig:depth}.
\begin{figure}[htb]
        \includegraphics[width=0.48\textwidth]{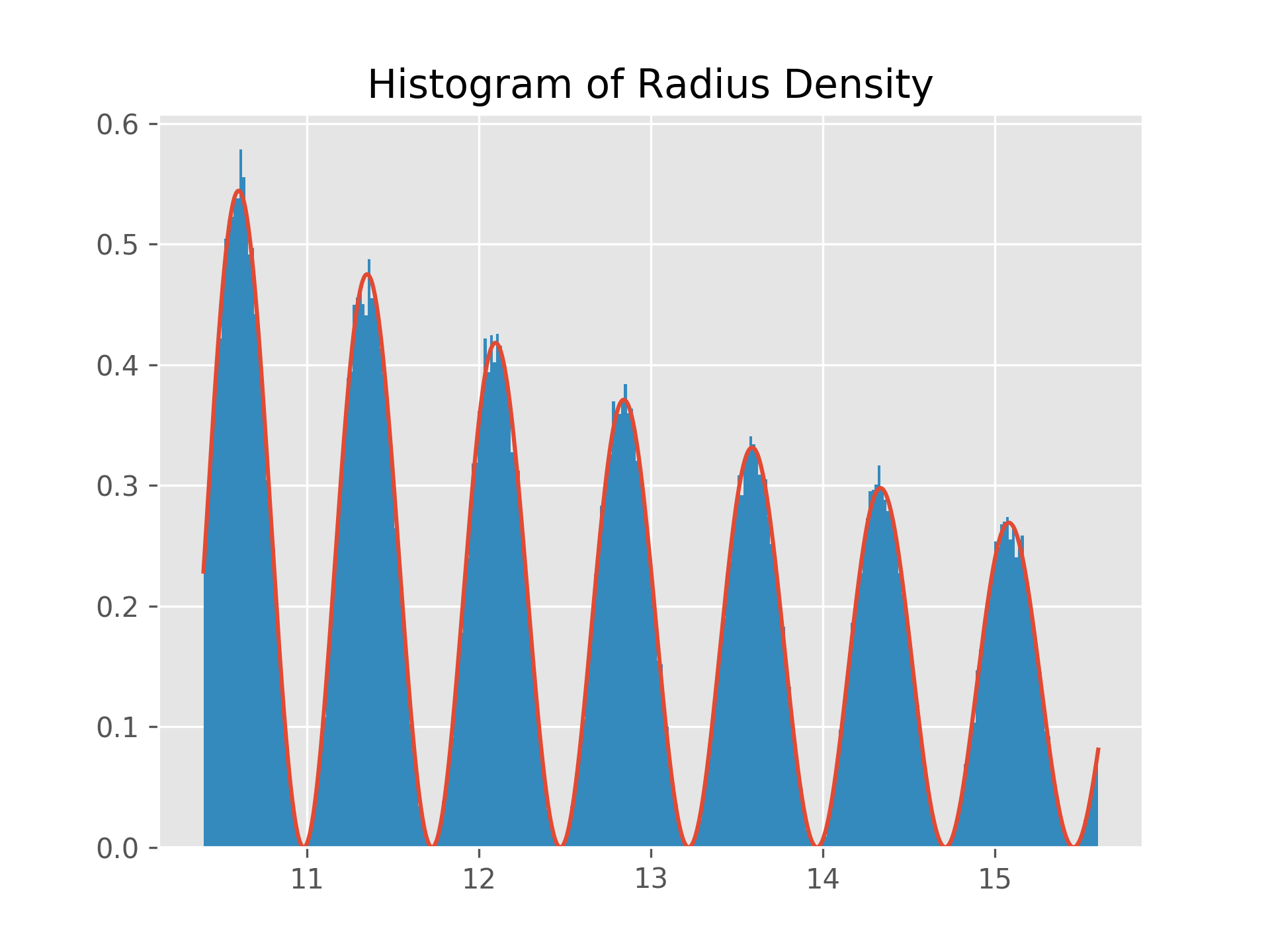}
        \hfill%
        \includegraphics[width=0.48\textwidth]{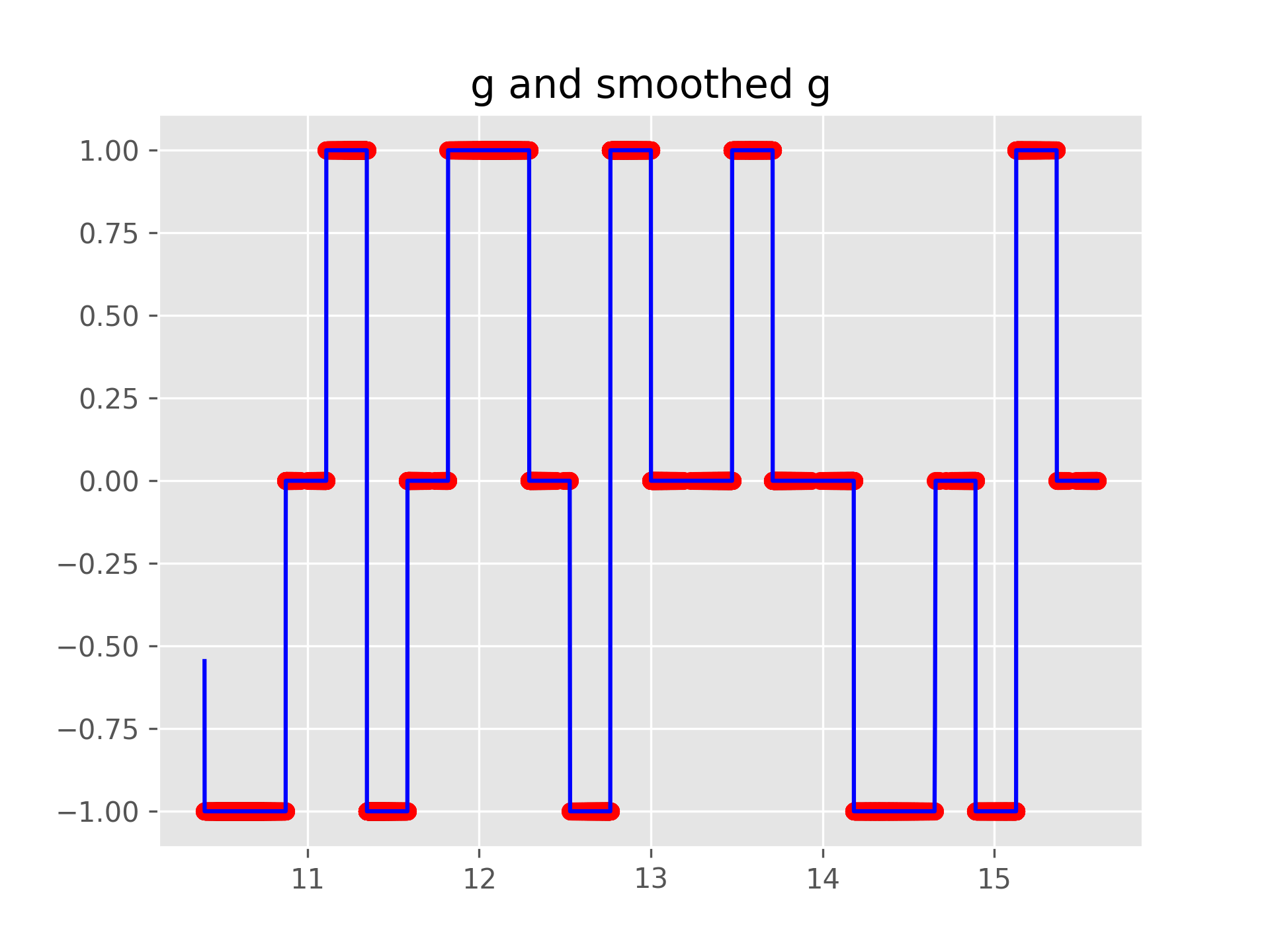}%
    \caption{Illustration of distributions of synthetic data for depth separation experiment. Left: radial density of data. Right: Smoothed (blue) and unsmoothed (red) target labels against the radius.}
  \label{fig:eldan-data}
\end{figure}

From Figure~\ref{fig:eldan-result} and Figure~\ref{fig:eldan-predict}, we can see that
the performance of the three models in both clasification and regression tasks
over Eldan and Shamir's example is consistent with that of Daniely's example.
\begin{figure}[htb]
        \includegraphics[width=0.48\textwidth]{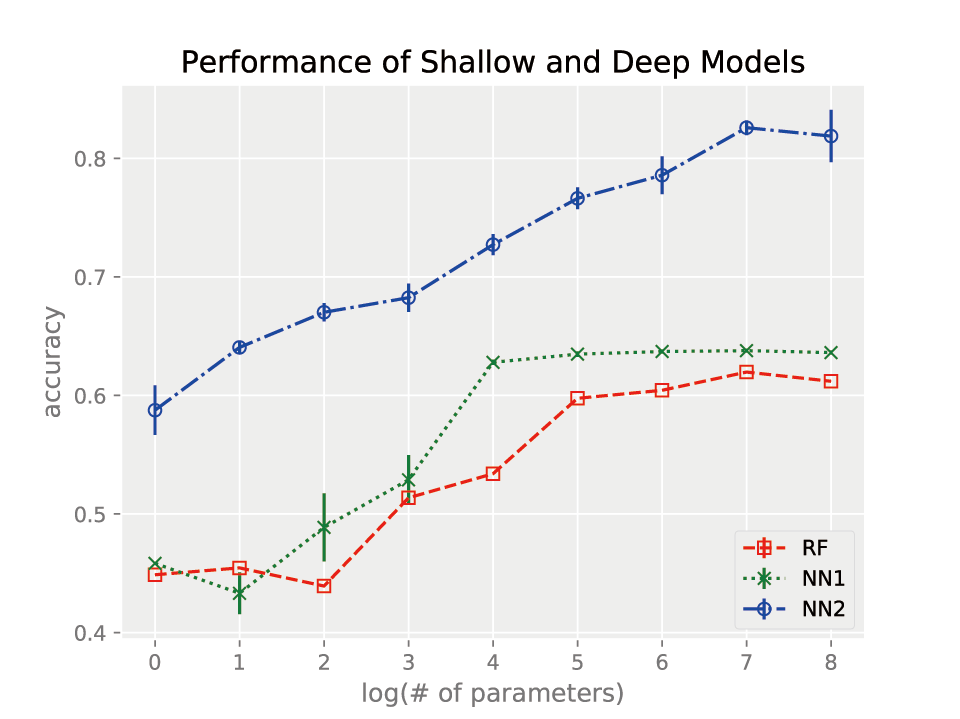}
        \hfill%
        \includegraphics[width=0.48\textwidth]{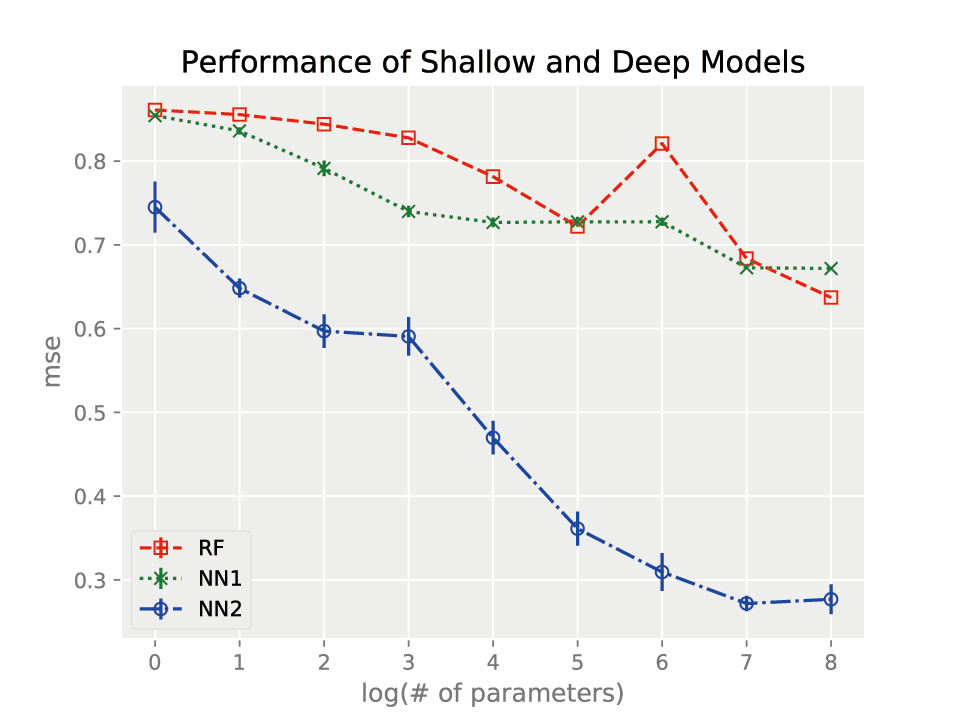}%
    \caption{Left: performance of the deep and shallow models in the classification task. Right: performance of the deep and shallow models in the regression task.}
  \label{fig:eldan-result}
\end{figure}

\begin{figure}[htb]
        \includegraphics[width=0.33\textwidth]{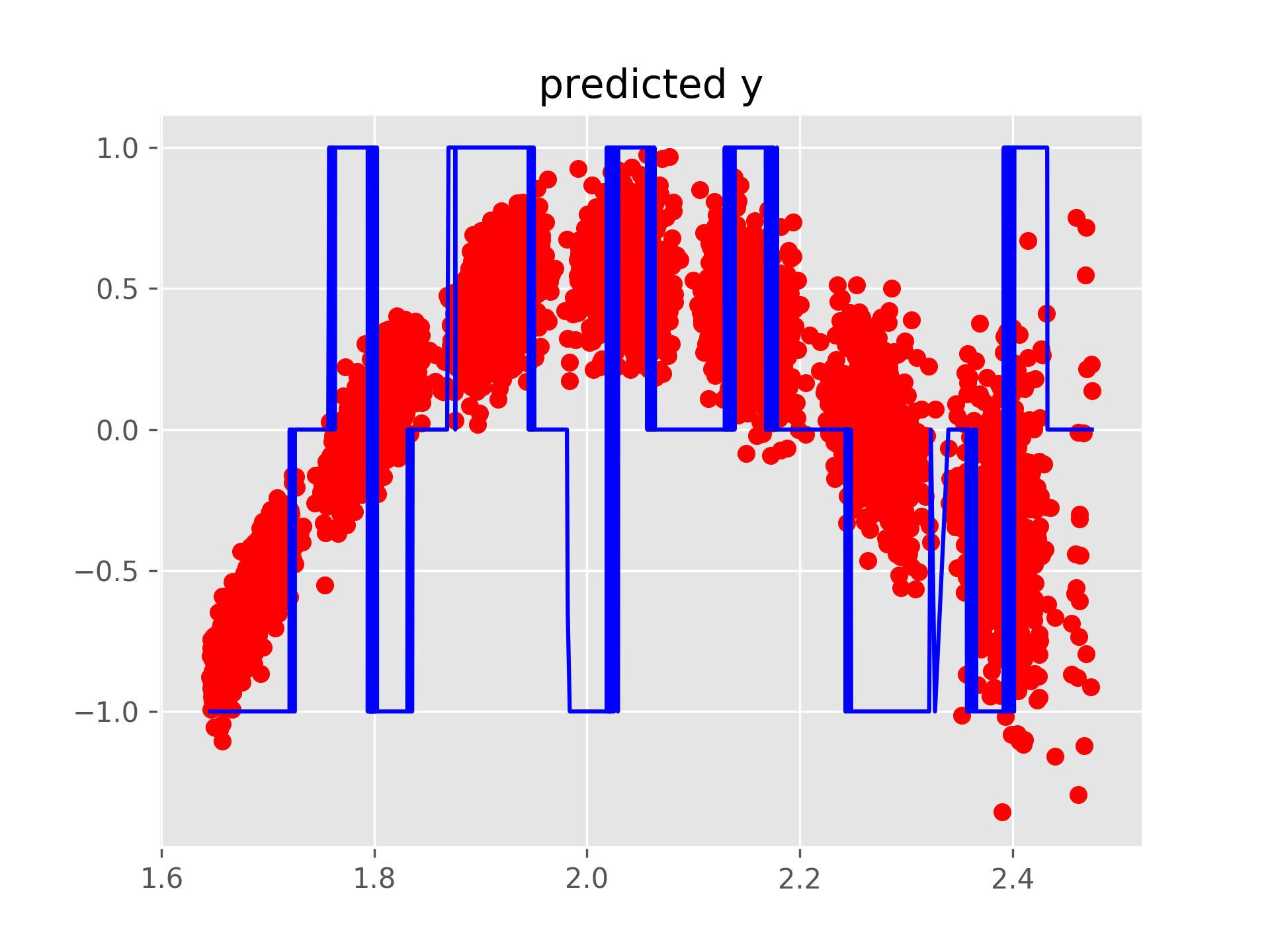}
        \hfill%
        \includegraphics[width=0.33\textwidth]{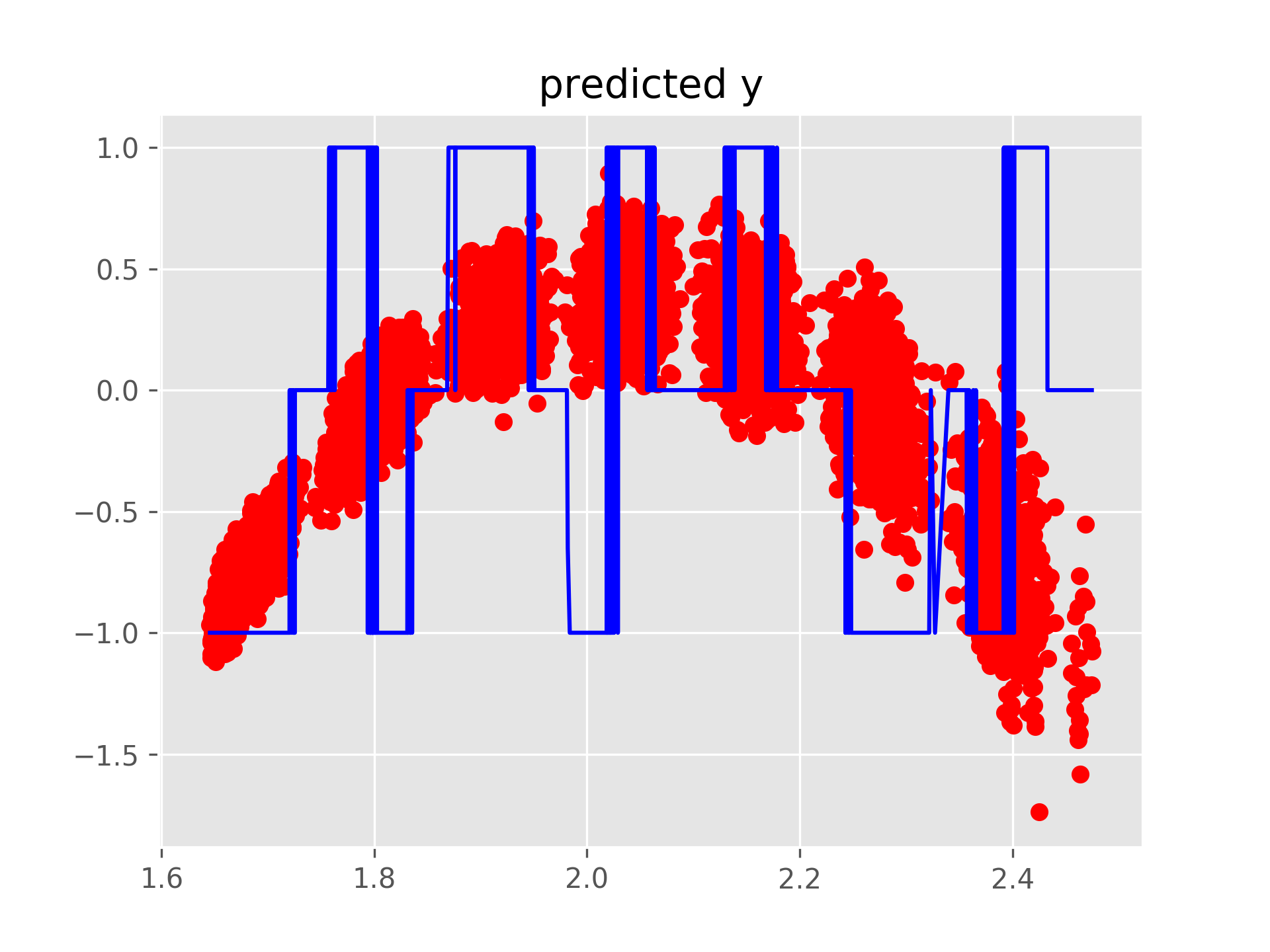}%
        \hfill%
        \includegraphics[width=0.33\textwidth]{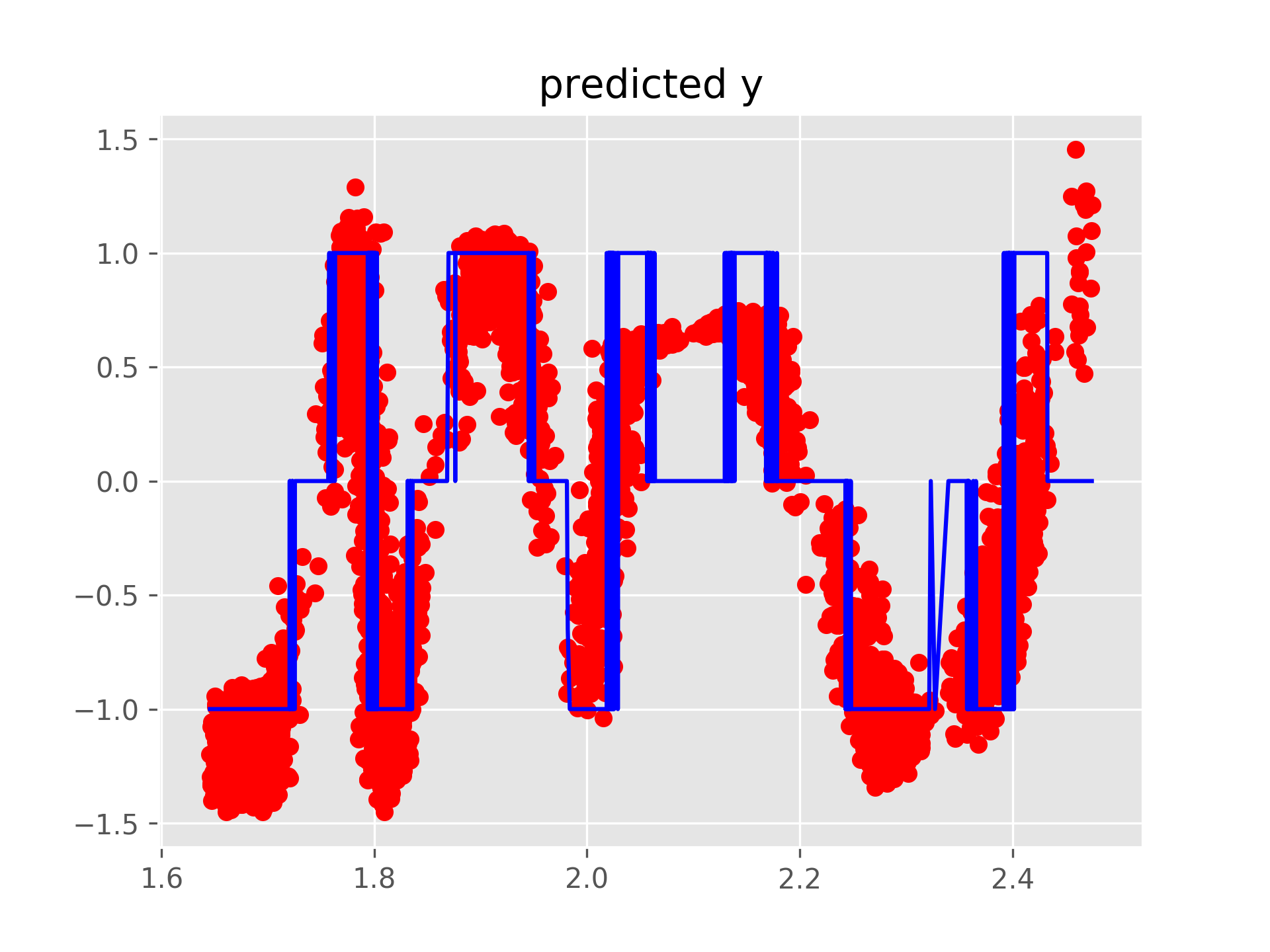}%
    \caption{Left: the predicted labels (red) by the best random ReLU features model compared to the true labels (blue) plotted against the radius $|x|$. Middle: the predicted labels (red) by the best 2-layer neural nets compared to the true labels (blue). Right: the predicted labels (red) by the best 3-layer neural nets compared to the true labels (blue).}
  \label{fig:eldan-predict}
\end{figure}